\documentclass[12pt]{elsarticle}

\usepackage{amsthm}
\usepackage{amssymb}
\usepackage[english]{babel}
\usepackage[letterpaper,top=2cm,bottom=2cm,left=3cm,right=3cm,marginparwidth=1.75cm]{geometry}

\usepackage{amsmath}
\usepackage[title]{appendix}
\usepackage{amsfonts}
\usepackage{graphicx}
\usepackage[colorlinks=true, allcolors=blue]{hyperref}
\usepackage{multirow}
\usepackage{multicol}

\newtheorem{theorem}{Theorem}

\newtheorem{corollary}{Corollary}
\newtheorem{lemma}{Lemma}

\begin{document}

\title{Improving the Performance of Echo State Networks Through State Feedback}
\author[1]{Peter J.~Ehlers}
\ead{ehlersp@arizona.edu}
\author[2]{Hendra I.~Nurdin}
\ead{h.nurdin@unsw.edu.au}
\author[1]{Daniel Soh}
\ead{danielsoh@optics.arizona.edu}

\affiliation[1]{organization={Wyant College of Optical Sciences, University of Arizona},
city={Tuscon, Arizona},
country={USA}}
\affiliation[2]{organization={School of Electrical Engineering and Telecommunications, University of New South Wales},
city={Sydney},
country={Australia}}

\begin{abstract}
 Reservoir computing, using nonlinear dynamical systems, offers a cost-effective alternative to neural networks for complex tasks involving processing of sequential data, time series modeling, and system identification. Echo state networks (ESNs), a type of reservoir computer, mirror neural networks but simplify training. They apply fixed, random linear transformations to the internal state, followed by nonlinear changes. This process, guided by input signals and linear regression, adapts the system to match target characteristics, reducing computational demands. A potential drawback of ESNs is that the fixed reservoir may not offer the complexity needed for specific problems. While directly altering (training) the internal ESN would reintroduce the computational burden, an indirect modification can be achieved by redirecting some output as input. This feedback can influence the internal reservoir state, yielding ESNs with enhanced complexity suitable for broader challenges. In this paper, we demonstrate that by feeding some component of the reservoir state back into the network through the input, we can drastically improve upon the performance of a given ESN. We rigorously prove that, for any given ESN, feedback will almost always improve the accuracy of the output. For a set of three tasks, each representing different problem classes, we find that with feedback the average error measures are reduced by $30\%-60\%$. Remarkably, feedback provides at least an equivalent performance boost to doubling the initial number of computational nodes, a computationally expensive and technologically challenging alternative. These results demonstrate the broad applicability and substantial usefulness of this feedback scheme.
\end{abstract}


\maketitle

\section{Introduction}
Compared to recurrent neural networks where excellent performance could only be obtained with very computationally expensive system adjustment procedures, the premise of reservoir computing is to use a \emph{fixed}  nonlinear dynamical system of the form \eqref{eq:states}-\eqref{eq:output} to perform signal processing tasks: 
\begin{align}
x_{k+1} &= f(x_k,u_k), \label{eq:states}\\
\hat{y}_k &= W^\top x_k + C, \label{eq:output}
\end{align}
where $u_k$ is the input signal at time $k$ \cite{lukovsevivcius2009reservoir,schrauwen2007overview,tanaka2019recent}.  The output $\hat{y}_k$ of the dynamical system is then typically taken as a simple linear combination of the states (nodes) $x$ of the dynamical system as given in \eqref{eq:output} plus some constant $C$ (as the output bias), where $W$ is a weight matrix and $W^\top$ is its transpose.  The nodes correspond to a basis map that is used to approximate an unknown map which maps input (discrete-time) sequences to output sequences that are to be learned by the dynamical system. Using the linear combination of states makes the training extremely straightforward and efficient as the weight matrix $W$ can be determined by a simple linear regression.

Reservoir computers (RCs) have been extensively used to predict deterministic sequences, in particular chaotic sequences, and data-based chaotic system modelling, see, e.g., \cite{JH04,pathak2018model,Rafayelyan20}. In the deterministic setting they have found applications in channel equalization \cite{JH04}, chaos synchronisation and encryption \cite{antonik2018using}, and model-free observers for chaotic systems \cite{LPHGBO17}. RCs have also been studied for the modelling of stochastic signals and systems with applications including time series modelling, forecasting, filtering and system identification \cite{grigoryeva2016reservoir,Grigoryeva:2018,gonon2019reservoir,Chen:2022}.  

Physical reservoir computing employs a device with complex temporal evolution, tapping into the computational power of a nonlinear dynamical system without extensive parameter optimization needed in typical neural networks. Inputs are fed into a reservoir, a natural system with complex dynamics, influencing its state based on current and past inputs due to its (limited) memory. The reservoir, running automatically, is altered only by these inputs. This approach models complex, nonlinear functions with minimal requirements: problem-related inputs and a linear fitting algorithm. Various physical platforms have been experimentally demonstrated for reservoir computing include, for instance, photonics \cite{larger2017high,Rafayelyan20}, spintronics \cite{torrejon2017neuromorphic} and even quantum systems \cite{chen2020temporal, Suzuki22,Yasuda23}. For a review of physical reservoir computing and quantum reservoir computing, see, e.g., \cite{tanaka2019recent,nakajima2021reservoir,Mujal21,markovic2020quantum}.

An echo state network (ESN) \cite{JH04,jaeger2001echo,jaeger2007echo,lukovsevivcius2012practical} is a type of RC using an iterative structure for adding nonlinearity to inputs. It is similar to a recurrent neural network, except that the neural weights are fixed and optimization occurs only at the output layer. In an ESN, the reservoir state at time step $k$ is represented by vector $x_k$, equivalent to the neural outputs at step $k$. Each step involves applying a fixed linear transformation given by a matrix $A$ to $x_k$, then adding to it a vector $B$ times the input value $u_k$, forming a new vector $z_k$. The matrix $A$ represents the fixed neural weights, while the vector $B$ represents the biases. A nonlinear transformation on each element of $z_k$ generates $x_{k+1}$, akin to neuron outputs. The affine function $\hat{y}_k$ of $x_k$ given in \eqref{eq:output} is then fit to a target value sequence ${y_k}$, giving an output $\hat{y}_k\approx y_k$ that approximates the target system. We describe the ESN framework with further detail in Section \ref{sec:ESN}.

The main drawback of this approach is that any specific ESN is only going to be effective for a certain subset of problems because the transformations that the reservoir applies are fixed, so a specific reservoir will tend to modify the inputs in the same way, leading to a limited range of potential outputs. It has been shown in \cite{Grigoryeva:2018} that ESNs as a whole are universal, meaning that for any target sequence $\{y_k\}$ and a given input sequence $\{u_k\}$, there will be an ESN with specific choices of $A, B,$ that can approximate it to any desired accuracy. However, it may not be practically feasible to find a sufficiently accurate ESN for a particular problem of interest as it may require choosing an excessively and impractically large ESN, and one may have to settle for an ESN with a weaker performance instead.

There have been previous efforts related to the above issue. In the context of an autonomous ESN with no external driving input, the work \cite{SA09} introduces a number of architectures. One architecture includes adding a second auxiliary ESN network besides the principal ``generator" ESN. The auxiliary is fed by a tunable linear combination of some nodes of the generator ESN, while a fixed (non-tunable) linear combination of the nodes of the auxiliary ESN is fed back to the generator network. The same error signal at the output of the generator ESN is used to train both the output weight of the principal and the weights that connect the generator to the auxiliary. The weight update is done recursively through an algorithm called the First-Order, Reduced and Controlled Error (FORCE) learning algorithm, which is in turn based on the recursive least squares algorithm.  A second architecture does not use feedback but allows modification of the some internal weights of the ESN besides the output weight, as in a conventional recurrent neural network. The internal and output weights are also updated using the FORCE algorithm. In \cite{FBD20}, multiple ESNs with tunable output weights that are interconnected in a fixed feedforward architecture (with no feedback loops) are considered. A set of completely known but randomly generated ``surrogate'' ESNs are coupled according to some architecture and  trained by simulation (``in silico") using the backpropagation algorithm for artificial neural networks. The ``intermediate" signals generated at the output of each component ESN are then used to train the output weights of another set of random ESNs, representing the ``true" ESNs that will be deployed, in the same architecture. The output weights of the individual true ESNs can be trained by linear regression. In \cite{MJS07}, in the continuous time setting, it was shown that an affine nonlinear system given by the nonlinear ODE:
\begin{align*}
\dot{x}_i &= f(x_1,\ldots,x_n) + g(x_1,\ldots,x_n)v,\ i=1,2,\ldots,n,
\end{align*}
for scalar real functions $x_1,\ldots,x_n$ and $v$
is universal in the sense that it can exactly emulate any other $n$-th order ODE of the form
\begin{align*}
z^{(n)} &= G(z,z^{(1)},\ldots,z^{(n-1)}) + u,
\end{align*}
that is driven by a signal $u$, where $z$ is a scalar signal and $z^{(j)}$ denotes the $j$-th derivative of $z$ with respect to time. The emulation is achieved by appropriately choosing scalar-valued real functions $K$ and $h$ and setting $v=K(x,u)$ and $z=h(x)$, where $x$ is the column vector $x=(x_1,\ldots,x_n)^{\top}$, where $\top$ denotes the transpose of a matrix. Also, any system of $k$ higher order ODEs in $z$ of the form above can be emulated by using $k$ different feedback terms. The technical report \cite{Luko07} considers the additional linear feedback of an auxiliary signal $z$, in general as a separate input $u_z$ into the ESN besides $u$, through a randomly generated matrix $W_z$ as $u_z = W_z z$. The signal $z$ is taken to be some linear combination of $x$, $z=V^{\top}x$. To determine $V$, some strategies for setting a training signal  $z_{train}$ for $z$ are proposed. $V$ is then computed by linear regression to minimize the mean-square error $\sum_{k=k_0}^N \|z_{train,k}-V^{\top} x_k\|^2$ over the time interval from $k_0$ to $N$, independently of $W$.

In this work, unlike \cite{SA09}, we are interested in ESNs that are driven by external input and are required to be convergent (forget their initial condition). Notably, the ESN in \cite{SA09} could not be convergent because a limit cycle was used to generate a periodic output without any inputs in the ESN. When driven by an external input such networks may produce an output diverging in time. Also, while the study in \cite{SA09} is motivated by biological networks, our work is motivated by the applications of physical reservoir computing. In particular, we are interested in enhancing the performance of a fixed physical RC by adding a simple but tunable structure external to the computer. Also, in contrast to \cite{FBD20}, in this paper we do not consider multiple interconnected ESNs in a feedforward architecture and do not use surrogates, but train a single  ESN augmented with a feedback loop directly with the data. While our work is related to \cite{MJS07} as discussed above, we do not seek universal emulation, but consider a linear feedback $K(x)=V^{\top}x$ that only depends on the state $x$, but not the input $u$, to enhance the approximation ability of ESNs. Also, compared to \cite{Luko07}, herein we consider feeding back a linear combination of the state $V^{\top}x$ by adding it to $u$ (i.e., by the substitution $u + V^{\top}x \rightarrow u$) without any modifications to the internal ESN structure, and optimize the weights $V$ through gradient descent. We do not introduce an additional artificial training sequence to determine the feedback $V^{\top}x$. Moreover, we provide a mathematical analysis of the generic performance advantage enabled by our scheme.

The goal of this paper is to study the use of feedback in the context of ESNs and show that it  will improve the performance of ESNs in the overwhelming majority of cases. Our proposal is to feed a linear function of the reservoir state back into the network as input. That is, for some vector $V$, we change the input from $u_k$ to $u_k +V^\top x_k$. We then optimize $V$ with respect to the cost function to achieve a better fit to the target sequence $\{y_k\}$. This has the effect of changing the linear transformation $A$ that the reservoir performs on $x_k$ at each time step, allowing us to partially control how the reservoir state evolves without modifying the reservoir itself. This will in essence provide us with a wider range of possible outputs for any given ESN, and can provide smaller ESNs an accuracy boost that makes them comparable to larger ones. Thus, our new paradigm of ESNs with feedback generates a significant performance boost with minimal perturbation of the system. We offer a thorough proof of a broad theorem, confidently ensuring that almost any ESN will experience a performance enhancement when a feedback mechanism is implemented, making this new scheme of ESNs with feedback universally applicable.

The structure of the paper is as follows. In Section \ref{sec:theory}, we provide some background on reservoir computers and echo state networks and introduce our feedback procedure. In Section \ref{sec:proof}, we provide a proof of the superiority of ESNs with feedback. In Section \ref{sec:optimize}, we describe how the new parameters introduced by feedback are optimized. In Section \ref{sec:tasks}, we provide numerical results that demonstrate the effectiveness of feedback for several different representative tasks. Finally, in Section \ref{sec:conclude} we give our concluding remarks.

In this paper, we denote the transpose of a matrix $M$ as $M^\top$, with the same notation used for vectors. The $n\times n$ identity matrix is written as $\mathbb{I}_n$, while the zero matrix of any size (including any zero vector) is written as $\mathbf{0}$. We treat an $n$-dimensional vector as an $n\times 1$ rectangular matrix in terms of notation, and in particular the outer product of two vectors $v_1$ and $v_2$ is written as $v_1 v_2^\top$. The vector norm $||v||$ denotes the standard 2-norm $||v||_2 = \sqrt{v^\top v}$. For a sequence whose $k$th element is given by $a_k$, we denote the entire sequence as $\{a_k\}$. When finite, a sum of the elements of such a sequence is often written notationally as if they were a sample of some stochastic process. We write weighted sums of these sequences as determinstic ``expectation" values (averages), so that for a sequence $\{a_k\}$ with $N$ entries starting from $k=0$ we may write $\langle a\rangle = \frac{1}{N}\sum_{k=0}^{N-1}a_k$. We also define the mean of such a sequence as $\overline{a} = \langle a\rangle$ and its variance as $\sigma_a^2 = \langle (a - \overline{a})^2\rangle$. The expectation operator is denoted by $\mathbb{E}[\cdot]$, the expectation of a random variable $X$ is denoted by $\mathbb{E}[X]$ and the conditional expectation of a random variable $X$ given random variables $Y_1,\ldots,Y_m$ is denoted by $\mathbb{E}[X | Y_1,\ldots,Y_m]$. We will denote the input and the output sequences of the training data as $\{ u_k \} = \{ u_k \}_{k=1,\ldots,N}, \{ y_k \} = \{ y_k \}_{k=1,\ldots,N}$, respectively. 

\section{Theory of Reservoir Computing with Feedback}
\label{sec:theory}
\subsection{Reservoir Computing and Echo State Networks} \label{sec:ESN}

A general RC is described by the following two equations:
\begin{align}
    x_{k+1} &= f(x_k, u_k) \\
    \hat{y}_k &= h(x_k),
\end{align}
where $x_k$ is a vector representing the reservoir state at time step $k$, $u_k$ is the $k$th member of some input sequence, and $\hat{y}_k$ is the predicted output. The function $f(x, u)$ is defined by the reservoir and is fixed, but the output function $h(x)$ is fit to the target sequence $y_k$ by minimizing a cost function $S$. In practice we usually choose (for $N$ training data points)
\begin{align}
    h(x) = W^\top x + C \label{eq:h-def}\\
    S = \frac{1}{2N}\sum_{k=0}^{N-1}(y_k-\hat{y}_k)^2,\label{eq:cost-function}
\end{align}
where the scalar $C$ and vector $W$ are chosen to minimize $S$, so that the problem of fitting the output function to data is just a linear regression problem. This setup is what enables the simulation and prediction of complex phenomena with a low computational overhead, because the reservoir dynamics encoded in $f(x, u)$ are complex enough to get a nonlinear function of the inputs $\{u_k\}$ that can then be made to approximate $\{y_k\}$ using linear regression.

In order for a RC to work, it must obey what is known as the {\em (uniform) convergence property}, or {\em echo state property} \cite{Chen:2022, jaeger2001echo, Manjunath:2013}. It states that, for a reservoir defined by the function $f(x, u)$ and a given input sequence $\{u_k\}$ defined for all $k\in\mathbb{Z}$, there exists a unique sequence of reservoir states $\{x_k\}$ that satisfy $x_{k+1}=f(x_k, u_k)$ for all $k\in\mathbb{Z}$. The consequence of this property is that the initial state of the reservoir in the infinite past does not have any bearing on what the current reservoir state is. This consequence combined with the continuity of $f(x, u)$ leads to the fading memory property \cite{Boyd:1985}, which tells us that the dependence of $x_k$ on an input $u_{k_0}$ for $k>k_0$ must dwindle continuously to zero as $k-k_0$ tends to infinity. This means that any initial state dependence should become negligible after the RC runs for a certain amount of time, so that the RC is reusable and produces repeatable, deterministic results while also retaining some memory capacity for past inputs.

It has been shown \cite{Grigoryeva:2018, Tran:2019} that a given RC will have the uniform convergence property if the reservoir dynamics $f(x, u)$ are contracting, or in other words if it satisfies 
\begin{align}
    ||f(x_1, u) - f(x_2, u)|| \leq \epsilon ||x_1 - x_2||,
\end{align}
where $\epsilon$ is some real number $0<\epsilon<1$. The norm in this inequality is arbitrary (as all norms on finite-dimensional metric spaces have equivalent effects), but it is usually chosen to be the standard vector norm $||v||_2 = \sqrt{v^\top v}$. This ensures that all reservoir states $x$ will be driven toward the same sequence of states defined by the inputs $\{u_k\}$.

An ESN is a specific type of RC described above, with 
\begin{align}
    f(x, u) = g(A x + B u), \label{eq:ESN-model}
\end{align}
where $A$ and $B$ are a random but fixed matrix and vector, respectively, while $g(z)$ is a nonlinear function that acts on each component of its input $z$. Throughout the paper we will take the dimension of the state $x$ to be $n$, and the dimensions of $A$ and $B$ to be $n \times n$ and $n \times 1$, respectively. For the output of the ESN we have that $C$ is a real scalar and $W$ is a real column vector of dimension $n$. This design gives the ESN resemblance to a typical neural network, where the linear transformation $z_k = A x_k + B u_k$ defines the input into the array of neurons, with $A$ providing the weights and $B u_k$ providing a bias. The element-wise nonlinear function $g(z)$ gives the array of outputs of the neurons as a function of the weighted inputs. The choices of $g, A,$ and $B$ define a specific ESN, though in practice $g$ is often chosen to be one of a specific set of preferred functions such as the sigmoid $\sigma(z) = (1 + e^{-z})^{-1}$ or $\tanh(z)$ functions. In this work, we choose the sigmoid function for our numerical results.

The convergence of the ESN can be guaranteed by subjecting the matrix $A$ to the constraint that $A^\top A < a^2 \mathbb{I}_{n}$ for a constant $a>0$. In other words, the singular values of $A$ must all be strictly less than some number $a$ which is determined by $g(z)$. For the sigmoid function, we can use $a = 4$, while for the $\tanh$ function we use $a= 1$. This originates from proving that 
\begin{align}
    ||g(z_1) - g(z_2)|| \leq a^{-1} ||z_1 - z_2||
\end{align}
for all $z_1, z_2\in\mathbb{Z}$, so that the convergence inequality will always be satisfied as long as
\begin{align}
\label{eq:constraint}
    a^{-1} ||(A x_1 + B u) - (A x_2 + B u)|| = a^{-1} ||A (x_1 - x_2)|| \leq \epsilon ||x_1 - x_2||,
\end{align}
for some $0<\epsilon<1$. Note that this is a sufficient but not necessary condition, as there could be combinations of $A, B,$ and $\{u_k\}$ such that \begin{align}
    ||g(A x_1 + B u_k) - g(A x_2 + B u_k)|| \leq \epsilon ||x_1 - x_2||
\end{align}
for all $x_1, x_2,$ and $k$, but this singular value criterion is much easier to test and design for while still providing a large space of possible reservoirs to choose from.

We parameterize how well the output of the ESN matches the target data using the normalized mean-square error (NMSE). In a linear regression problem we can show that the mean-squared error is 
\begin{align}
\label{eq:MSE}
    \langle(y-\hat{y})^2\rangle &= \frac{1}{N}\sum_{k=0}^{N-1}(y_k-\hat{y}_k)^2 = 2S,
\end{align}
where we are averaging over the $N$ time steps corresponding to the training interval. With $\hat{y}_k = W^{\top} x_k + C$, we can show that
\begin{align}
    \langle(y-\hat{y})^2\rangle &= \langle(C+W^\top x-y)^2\rangle \\
    &= (C+W^\top \langle x\rangle-\langle y\rangle)^2+\langle(W^\top (x-\langle x\rangle)-(y-\langle y\rangle))^2\rangle \\
    &= (C+W^\top \langle x\rangle-\langle y\rangle)^2+W^\top K_{xx} W-2W^{\top} K_{xy}+\sigma_y^2,
\end{align}
where 
\begin{align}
\label{eq:Kdefn}
    K_{xx} &= \langle(x-\langle x\rangle)(x-\langle x\rangle)^\top\rangle = \langle x x^\top\rangle - \langle x\rangle\langle x\rangle^\top \\
    K_{xy} &= \langle(x-\langle x\rangle)(y-\langle y\rangle)\rangle = \langle x y\rangle - \langle x\rangle\langle y\rangle.
\end{align}
The values of $C$ and $W$ that minimize the mean-squared error are
\begin{align}
\label{eq:optimizeC}
    C = \langle y\rangle - W^\top \langle x\rangle \\
    W = K_{xx}^{-1} K_{xy}. \label{eq:optimize}
\end{align}
Note that since $K_{xx}$ is a covariance matrix, it must be positive semi-definite, but by inverting it to find the optimal value of $W$ we have further assumed that it is positive definite. This assumption is equivalent to saying that all of the vectors $x_k$ span the entire vector space $\mathbb{R}^{n_c}$, where $n_c$ is the dimension of $W$ and all $x_k$'s. This is reasonable because in practice we usually take the number of training steps $N>>n_c$, and since each $x_k$ is a nonlinear transformation of the previous one, it is unlikely that any vector $v$ will satisfy $v^\top x_k = 0$ for all $k\in{0,\dots,N-1}$. Nevertheless, in the event that $K_{xx}$ is not invertible, we can take the pseudoinverse of $K_{xx}$ instead. This is because the components of $W$ parallel to the zero eigenvectors of $K_{xx}$ are not fixed by the optimization (which is why the inversion fails in the first place), so we are free to choose those components to be zero, which makes Eq. \eqref{eq:optimize} correct when using the pseudoinverse of $K_{xx}$ as well.

Plugging the optimized values of $C$ and $W$ into the mean-squared error gives
\begin{align}
\label{eq:minexpr}
    \left(\langle(y-\hat{y})^2\rangle\right)_{\min} &= \sigma_y^2-K_{xy}^\top K_{xx}^{-1} K_{xy}.
\end{align}
From the original expression for the mean-squared error in Eq. \eqref{eq:MSE}, we can see that it is non-negative. Since $K_{xx}$ is a covariance matrix, the quantity $K_{xy}^\top K_{xx}^{-1} K_{xy} = W^\top K_{xx} W \ge 0$. Thus $\left(\langle(y_k-\hat{y}_k)^2\rangle\right)_{\min}$ is bounded above by $\sigma_y^2$, so we may define a normalized mean-squared error, or NMSE, by
\begin{align}
    \mathrm{NMSE} &= \frac{\langle(y-\hat{y})^2\rangle}{\sigma_y^2}.
\end{align}
This quantity is guaranteed to be between 0 and 1 for the training data, though it may exceed 1 for an arbitrary test data set. We can also see by Eqs. \eqref{eq:MSE} and \eqref{eq:minexpr}, the task of minimizing $S$ as a function of $C, W,$ and $V$ is equivalent to maximizing $K_{xy}^\top K_{xx}^{-1} K_{xy}$ as a function of $V$.

\subsection{ESNs with Feedback}
The main result of this work is the introduction of a feedback procedure to improve the performance of ESNs. We add an additional step to the process where the input is taken to be $u_k + V x_k$ at each time step as opposed to just $u_k$. The reservoir of the ESN is then described by
\begin{align}
    x_{k+1} &= g(A x_k + B (u_k + V^\top x_k)) = g((A + B V^\top) x_k + B u_k).
\end{align}
From this equation, we see that the feedback causes this ESN to behave like a different network that uses $\overline{A} = A + B V^\top$ as a transformation matrix instead of $A$. We achieve this without modifying the RC itself, using only the pre-existing input channel and the reservoir states that we are already measuring. This provides a practical way of changing the reservoir dynamics without any internal hardware modification. We then optimize for $V$ using batch gradient descent to further reduce the cost function $S$.

Note, however, that in attempting to modify $A$ we run the risk of eliminating the uniform convergence of the ESN. Thus, there must be a constraint placed on $V$ in order to keep the network convergent. In accordance with the constraint in Eq. \eqref{eq:constraint}, we require that $\overline{A}^\top \overline{A} < a^2 \mathbb{I}_{n}$ in addition to $A$, which places some limitations on the value of $V$. This constraint is generally quite complex to solve beyond this inequality, but it is possible to formulate this as a linear matrix inequality in $\overline{A}$; see, e.g., \cite[\S IV]{Chen:2022}. In addition, this condition can be easily applied during the process of optimizing $V$.

\section{Universal Superiority of ESN with Feedback over ESN without Feedback}
\label{sec:proof}

In this section, we prove our central theorem stating that the ESN with feedback accomplishes smaller overall errors than the ESN without feedback. For this, we start with a theorem for an individual ESN:

\begin{theorem}[Superiority of feedback for a given ESN and training data] \label{thm1}
    For any given matrix $A$ and vector $B$ in Eq. \eqref{eq:ESN-model}, and given sets of training inputs $\{u_k\}=\{u_k\}_{k=1,\ldots,N}$ and outputs $\{y_k\}=\{y_k\}_{k=1,\ldots,N}$ of finite length, define an optimized cost function $S_\mathrm{min} (A,B, \{u_k\}, \{y_k\})$ with appropriate optimal $W$ and $C$. Then, for almost any given $(A,B,\{u_k\}, \{y_k\})$ except for vanishingly small number of $(A,B, \{u_k\}, \{y_k\})$, the feedback always reduces the cost function further:
    \begin{equation}
        \min_{V} S_\mathrm{min} (A+B V^\top, B, \{u_k\}, \{y_k\}) < S_\mathrm{min} (A,B, \{u_k\}, \{y_k\}). \label{eq:thm1-inequality}
    \end{equation} 
    Moreover, if $A$ is such that $A^{\top}A < a^2 \mathbb{I}_{n}$, where $a$ is a constant that guarantees that the ESN is convergent, then the feedback gain $V$ can alawys be chosen such that the ESN with feedback is also convergent and satisfy the above.
\end{theorem}

\subsection{Preliminary Definitions and Relations}
To prove Theorem \ref{thm1}, we will set up a number of lemmas and definitions prior to starting the main proof. This preliminary work will primarily concern the cases in which Eq. \eqref{eq:thm1-inequality} does not hold, and the lemmas will show that the number of such cases is vanishingly small. The main proof of Theorem \ref{thm1} will then prove the strict inequality for all other cases. The following rigorously proves that the number of cases for $(A,B,\{u_k\}, \{y_k\})$ that satisfies the above is vanishingly small. 

We will need a number of new symbols and definitions to facilitate the proofs of Theorem \ref{thm1} and the following lemmas. For a given RC $(A,B)$ and training data set $(\{u_k\}, \{y_k\})$, there are several cases where the change in the vector $V$ may be zero. Consider an ESN with a specific choice of the matrix $A$, vector $B$, and nonlinear function $\sigma(z)$. Also consider a fixed input sequence $\{u_k\}$, and to train our network we will use the $N$ time steps ranging from 0 to $N-1$. To see how the derivative of the minimized cost function $S_{\mathrm{min}}$ with respect to the feedback parameters $V$ can be zero, define the matrix $X_{ik} = \frac{1}{\sqrt{N}}(x_{k,i}-\overline{x}_i)$ for time steps $k$ in the training data set. This is similar to the procedure used in \cite{Fujii:2017} to optimize for $W$. In other words, with $n_c + 1$ computational nodes ($n_c$ coming from the vector $W$ and 1 from $C$) and $N>n_c$ training data points, the matrix $X$ is an $n_c \times N$ rectangular matrix whose columns are proportional to the mean-adjusted reservoir state $x_k - \overline{x}$ at each time step $k$ in the training set. Further, define the vector $Y_k = \frac{1}{\sqrt{N}}y_k$. With these definitions, we can rewrite the quantities previously defined in the context of Eq. \eqref{eq:Kdefn} as
\begin{align}
    K_{xy} &= \frac{1}{N}\sum_{k=0}^{N-1}(x_k-\overline{x})(y_k-\overline{y}) = \frac{1}{N}\sum_{k=0}^{N-1}(x_k-\overline{x})y_k - \mathbf{0} = X Y \label{eq:K_xy} \\
    K_{xx} &= \frac{1}{N}\sum_{k=0}^{N-1}(x_k-\overline{x})(x_k-\overline{x})^\top = X X^\top.
\end{align}
Here, the second equality of Eq. \eqref{eq:K_xy} used the fact that $\frac{1}{N}\sum_{k=0}^{N-1} (x_k - \bar{x}) \bar{y} = \bar{x} \bar{y} - \bar{x} \bar{y} = \mathbf{0}$.

Denote the pseudoinverse of $X$ as $X^{-1}$. Note that while $X X^{-1}$ is the $n_c\times n_c$ identity matrix, $X^{-1} X$ is not the $N\times N$ identity matrix in the vector space of time steps denoted by $k$. Instead, it is a projection operator we will call $\Pi_x$. The singular value decomposition of $X$ is given by $X = U_{n_c}\Sigma U_N^\top$, where $U_{n_c}$ is an $n_c\times n_c$ orthogonal matrix, $\Sigma$ is taken to be an $n_c\times N$ rectangular diagonal matrix with non-negative values, and $U_N$ is an $N\times N$ orthogonal matrix. The pseudoinverse of $X$ is defined to be $X^{-1} \equiv U_N\Sigma^{-1} U_{n_c}^\top$, where the pseudoinverse of $\Sigma$ is defined so that with $\Sigma_{jk} = \sigma_j\delta_{jk}$ we have $\Sigma^{-1}_{kj} = \sigma_j^{-1}\delta_{jk}$. This also implies that $(X^{-1})^\top = U_{n_c}(\Sigma^{-1})^\top U_N^\top = (X^\top)^{-1}$ since $(\Sigma^{-1})^\top = (\Sigma^\top)^{-1}$. 

The product of $\Sigma^{-1}$ and $\Sigma$ is given by $\Sigma^{-1} \Sigma = \Pi_{n_c}$, where the elements of $\Pi_{n_c}$ are defined by
\begin{align}
    (\Pi_{n_c})_{kl} &\equiv \theta_{-}(n_c-k)\delta_{kl},
\end{align}
where $\theta_{-}(x)$ is the step function with $\theta_{-}(0) = 0$. Note that this is a projection operator since it satisfies $\Pi_{n_c} \Pi_{n_c} = \Pi_{n_c}$. Thus the product of $X^{-1}$ and $X$ is given by
\begin{align}
    X^{-1} X = (U_N\Sigma^{-1} U_{n_c}^\top) (U_{n_c}\Sigma U_N^\top) = U_N\Pi_{n_c} U_N^\top \equiv \Pi_x.
\end{align}
$\Pi_x$ must also be a projection operator since $\Pi_x \Pi_x = U_N\Pi_{n_c} \Pi_{n_c} U_N^\top = U_N\Pi_{n_c} U_N^\top = \Pi_x$. We also see that $\Pi_x$ is symmetric since $\Pi_{n_c}$ is symmetric, so $(X^{-1} X)^\top = X^\top (X^\top)^{-1} = \Pi_x$ as well. This method of defining a singular value decomposition of the $X$ matrix and obtaining the corresponding projection matrix $\Pi_x$ is similar to the methods used to obtain theoretical results in \cite{Kubota:2021, Nakajima:2019}. Note that the inversion of $\Sigma$ assumes that all singular values are nonzero, but we already make this assumption when optimize for $W$. By the expression for $W$ in Eq. \eqref{eq:optimize} and the discussion following that equation, this assumption to reasonable.

To get an expression for $S_{\mathrm{min}}$ (short for $S_{\mathrm{min}}(A+B  V^\top,B,\{u_k\}, \{y_k\})$) in this formalism, define the $N$-dimensional vector $\hat{e}$ such that its elements are given by $\hat{e}_k = \frac{1}{\sqrt{N}}$. It can then be shown that $\hat{e}^\top Y = \frac{1}{N}\sum_{k=0}^{N-1}y_k= \overline{y}$. Then the variance of $y_k$ can be written as $\sigma_y^2 = \frac{1}{N}\sum_{k=0}^{N-1}y_k^2 - \overline{y}^2 = Y^\top(\mathbb{I}_N - \hat{e}\hat{e}^\top) Y$, where $\mathbb{I}_n$ is the identity matrix of dimension $n\times n$. From this expression and the expression for twice the optimal cost given in Eq. \eqref{eq:minexpr}, $S_{\mathrm{min}}$ is then
\begin{align}
    S_{\mathrm{min}} &= \frac{1}{2}\left(\sigma_y^2 - K_{xy} K_{xx}^{-1} K_{xy}\right) \\
    &= \frac{1}{2}Y^\top\left(\mathbb{I}_N  - \hat{e}\hat{e}^\top - X^\top (X^\top)^{-1} X^{-1} X\right) Y.
\end{align}

With the projection operator $\Pi_x$, we can rewrite this expression for optimal cost function as
\begin{align}
    S_{\mathrm{min}} = \frac{1}{2}Y^\top\left(\mathbb{I}_N - \hat{e}\hat{e}^\top - \Pi_x \Pi_x\right) Y = \frac{1}{2}Y^\top\left(\mathbb{I}_N - \hat{e}\hat{e}^\top - \Pi_x\right) Y.
\end{align}
Thus the effect of indirectly modifying the RC with feedback is to shift the basis of the projection operator $\Pi_x$ to have as large of an overlap with the target sequence $Y$ as possible. The derivative of $S_{\mathrm{min}}$ with respect to some general parameter $\theta$ of the RC is then given simply by 
\begin{align}
    \frac{dS_{\mathrm{min}}}{d\theta} = -\frac{1}{2}Y^\top\frac{d\Pi_x}{d\theta} Y.
\end{align}
The target $Y$ is independent of the RC and thus independent of $\theta$, so any changes to $S_{\mathrm{min}}$ as a result of changing $\theta$ must come from a change in $\Pi_x$. The fact that $\Pi_x$ is a projection operator of rank $n_c$ tells us some properties of any of its derivatives. First, from the property $\Pi_x \Pi_x = \Pi_x$ we get 
\begin{align}
\label{eq:dproj}
    \frac{d\Pi_x}{d\theta} = \frac{d}{d\theta}(\Pi_x \Pi_x) = \frac{d\Pi_x}{d\theta} \Pi_x + \Pi_x \frac{d\Pi_x}{d\theta}.
\end{align}
Note that this implies $\Pi_x \frac{d\Pi_x}{d\theta} \Pi_x = 2(\Pi_x \frac{d\Pi_x}{d\theta} \Pi_x)$. The only matrix that is equal to 2 times itself is the zero matrix, so $\Pi_x \frac{d\Pi_x}{d\theta} \Pi_x$ must be the zero matrix. 

\subsection{Lemmas for Proving the Lower Dimensionality of Cases where \texorpdfstring{$\nabla_V S_{\mathrm{min}} = \mathbf{0}$}{TEXT}}
Now that we have established that the dependence of the cost function on the reservoir is entirely determined by a projection matrix $\Pi_x$, we are ready to begin discussing cases where $\frac{dS_{\mathrm{min}}}{d\theta} = 0$.
\begin{lemma}[Categorization of cases where a derivative of $S_{\mathrm{min}}$ w.r.t. a general reservoir parameter $\theta$ vanishes]
    Given $S_{\mathrm{min}}(A, B, \{u_k\},\{y_k\})$ and any parameter $\theta$ that the reservoir is dependent on, the cases where $\frac{dS_{\mathrm{min}}}{d\theta}=0$ fall into one of two categories, one where $\frac{dS_{\mathrm{min}}}{d\theta}=0$ only for specific target sequences $\{y_k\}$ and one where $\frac{d\Pi_x}{d\theta}=0$. Furthermore, the former category is divided into 3 more categories in which $\Pi_x Y = \mathbf{0}$, $\Pi_x Y = Y$, or neither. 
\end{lemma}

\begin{proof}
For the discussion that follows, define the vector subspace $\mathcal{Y}_{\parallel}$ to be the space of $N$-dimensional vectors with real coefficients such that $\mathcal{Y}_{\parallel} = \{Y|Y\in\mathbb{R}^N,\Pi_x Y = Y\}$. Define also the vector subspace $\mathcal{Y}_{\perp}$ such that $\mathcal{Y}_{\perp} = \{Y|Y\in\mathbb{R}^N,\Pi_x Y =\mathbf{0}\}$. Note that it is always possible to construct an orthonormal basis of vectors $\hat{Y}_k$ in $\mathbb{R}^N$ where the first $n_c$ basis vectors are in $\mathcal{Y}_{\parallel}$ and the remaining $N-n_c$ basis vectors are in $\mathcal{Y}_{\perp}$. This makes it useful to define
\begin{align}
\label{eq:Tdef}
    T_\theta &\equiv \Pi_x \frac{d\Pi_x}{d\theta}(\mathbb{I}_N-\Pi_x) = \Pi_x \frac{d}{d\theta}\left(X^{-1} X\right)(\mathbb{I}_N-\Pi_x) \\
    &= \Pi_x X^{-1} \frac{dX}{d\theta}(\mathbb{I}_N-\Pi_x) + \Pi_x \frac{dX^{-1}}{d\theta} X(\mathbb{I}_N-\Pi_x) \\
\label{eq:Tdef2}
    &= X^{-1} \frac{dX}{d\theta}(\mathbb{I}_N-\Pi_x),
\end{align}
where in the last line we used $\Pi_x  X^{-1}= X^{-1} X X^{-1} = X^{-1} \mathbb{I}_{n_c} = X^{-1}$ to simplify the first term and $X \Pi_x = X X^{-1} X = \mathbb{I}_{n_c} X = X$ to eliminate the second term. This definition is useful because $\Pi_x \frac{d\Pi_x}{d\theta} \Pi_x = \mathbf{0}$, so from the definition above $T_\theta = \Pi_x \frac{d\Pi_x}{d\theta}$, and therefore from Eq. \eqref{eq:dproj} we have
\begin{align}
    \frac{d\Pi_x}{d\theta} &= T_\theta + T_\theta^\top.
\end{align}
This also implies that 
\begin{align}
\label{eq:DSmin}
    \frac{dS_{\mathrm{min}}}{d\theta} &= -\frac{1}{2}Y^\top(T_\theta + T_\theta^\top) Y = -Y^\top T_\theta Y,
\end{align}
so the change in $S_{\mathrm{min}}$ depends entirely upon $T_\theta$ with respect to $Y$.

From the definition of $T_\theta$ in Eq. \eqref{eq:Tdef}, because there is a $\Pi_x$ on the left side of the matrix, we then have that $Y^\top T_\theta = \mathbf{0}$ for all $Y\in\mathcal{Y}_{\perp}$. Since $\mathcal{Y}_{\perp}$ has dimension $N-n_c$, there must be at least $N-n_c$ zero singular values of $T_\theta$. Let $m_\theta$ be the matrix rank of $T_\theta$ (number of nonzero singular values), which by the previous argument cannot be larger than $n_c$. Then the singular value decomposition of $T_\theta$ can be written as
\begin{align}
    T_\theta = \sum_{j=0}^{m_\theta-1}\sigma_{\theta,j}\hat{y}_{\theta,j}^{\parallel}\left(\hat{y}_{\theta,j}^{\perp}\right)^\top,
\end{align}
where $\sigma_{\theta,j}$ is a strictly positive singular value of $T_\theta$, $\hat{y}_{\theta,j}^{\parallel}$ is one of $m_\theta$ orthonormal basis vectors in $\mathcal{Y}_{\parallel}$, and $\hat{y}_{\theta,j}^{\perp}$ one of $m_\theta$ orthonormal basis vectors in $\mathcal{Y}_{\perp}$. The reason the right basis vectors are in $\mathcal{Y}_{\perp}$ is because of the $(\mathbb{I}_N-\Pi_x)$ on the right side of Eq. \eqref{eq:Tdef}, which makes to so that $T_\theta Y = \mathbf{0}$ for all $Y\in\mathcal{Y}_{\parallel}$.

With this decomposition of $T_\theta$, we can use Eq. \eqref{eq:DSmin} to rewrite $\frac{dS_{\mathrm{min}}}{d\theta}$ as
\begin{align}
    \frac{dS_{\mathrm{min}}}{d\theta} &= -\sum_{j=0}^{m_\theta-1}\sigma_{\theta,j}c_{\theta,j}^{\parallel}\left(c_{\theta,j}^{\perp}\right)^\top,
\end{align}
where $c_{\theta,j}^{\parallel} = Y^\top \hat{y}_{\theta,j}^{\parallel}$ and $c_{\theta,j}^{\perp} = Y^\top \hat{y}_{\theta,j}^{\perp}$. There are 3 broad categories of $Y$ for which $\frac{dS_{\mathrm{min}}}{d\theta}$ vanishes for a given $T_\theta$:
\begin{itemize}
    \item $Y$ is orthogonal to every $\hat{y}_{\theta,j}^{\parallel}$, or equivalently $\Pi_x Y = 0$.
    \item $Y$ is orthogonal to every $\hat{y}_{\theta,j}^{\perp}$, or equivalently $\Pi_x Y = Y$.
    \item Neither of the above statements are true, but the coefficients $c_{\theta,j}^{\parallel}$ and $c_{\theta,j}^{\perp}$ are such that the sum $\sum_{j=0}^{m_\theta-1}\sigma_{\theta,j}c_{\theta,j}^{\parallel}c_{\theta,j}^{\perp}$ vanishes.
\end{itemize}
There is also the possibility that $m_\theta$ is zero, meaning that every singular value of $T_\theta$ is zero, so that $\frac{dS_{\mathrm{min}}}{d\theta} = 0$ for any $Y$.
\end{proof}

In what follows, while we cannot rule out any of these possibilities for the feedback vector $V$, we can show that the space of $Y$'s that fit into the above three criterion are of lower dimension than the general space of $N$-dimensional vectors that encompasses all $Y$'s, and give criterion for numerically testing whether any of these cases hold for a given reservoir computation. Also, in the event that the gradient of $S_{\mathrm{min}}$ with respect to $V$ vanishes for any $Y$, we will show that the number of solutions in the space of possible matrices $A$, vectors $B$, and input sequences $u_k$ is of lower dimension as well with testable criterion for a given ESN.

\begin{lemma}[Lower dimensionality of cases where $\nabla_V S_{\mathrm{min}} = \mathbf{0}$ while $\nabla_V \Pi_x \neq \mathbf{0}$]
\label{lemmapt1}
    Given a specific ESN defined by a matrix $A$, vector $B$, and input sequence $\{u_k\}$ such that the projection operator $\Pi_x$ satisfies $\nabla_V \Pi_x \neq \mathbf{0}$, the space of training vectors $Y$ that then leads to $\nabla_V S_{\mathrm{min}} = \mathbf{0}$ is of lower dimension than the space of all training vectors, whose dimension is $N$.
\end{lemma}

\begin{proof}
Define $T_i \equiv T_{V_i}$ to be the same as in Eq. \eqref{eq:Tdef} with $\theta$ replaced with $V_i$ for each $i=0,\dots,n_c-1$. In order for the gradient $\nabla_V S_{\mathrm{min}}$ to be zero, we require that $\frac{dS_{\mathrm{min}}}{dV_i} = - Y^\top T_i Y$ be zero for \textit{all} $i$'s. This means that for a given set of $T_i$'s, there are three cases in which $\nabla_V S_{\mathrm{min}} = \mathbf{0}$ because of the particular form of $Y$:
\begin{itemize}
    \item[Case 1:] Let $\mathcal{Y}_{V}^{\parallel}$ be the span of the set of vectors that contains $\hat{y}_{i,j}^{\parallel}$ for every $i\in\{0,\dots,n_c-1\}$ and $j\in\{0,\dots,\mathrm{rank}(T_i)-1\}$, and let its dimension be $m_\parallel$. Then if $Y$ is such that $Y^\top y^\parallel = 0$ for all $y^\parallel\in\mathcal{Y}_{V}^{\parallel}$, then $\nabla_V S_{\mathrm{min}} = \mathbf{0}$ because $Y^\top T_i = \mathbf{0}$ for all $i$. This includes the case where $S_{\mathrm{min}}$ is at its maximum possible value of $\frac{1}{2}\sigma_y^2$, in which the RC has utterly failed to capture any properties of $Y$. The dimension of the space of $Y$'s that fall under this case is $N-m_\parallel$. This is because $\mathcal{Y}_{V}^{\parallel}$ is spanned by $m_\parallel$ basis vectors, so the space of $Y$'s that are orthogonal to all of them is spanned by the remaining $N-m_\parallel$ basis vectors. We can calculate $m_\parallel$ from the matrix defined as
    \begin{align}
        M_\parallel = \sum_{i=0}^{n_c-1} T_i T_i^\top.
    \end{align}
    $m_\parallel$ is given by the rank of $M_\parallel$ because each $T_i T_i^\top$ is a positive semi-definite matrix, so the only way that $Y^\top M_\parallel = \mathbf{0}$ is if $Y^\top T_i = \mathbf{0}$ for all $i$. The dimension of the space of vectors that satisfy this relation is $N-m_\parallel$ as mentioned previously, so if $M_\parallel$ has $N-m_\parallel$ zero eigenvalues, then that leaves $m_\parallel$ nonzero eigenvalues.
    
    We see from Eq. \eqref{eq:Tdef2} that computing $M_\parallel$ will involve calculating $X^{-1}$, but since we are only interested in the rank of the matrix we can find an alternative. Recall that $X$ is an $n_c\times N$ matrix whose singular values are strictly positive, so therefore the rank of $X$ is $n_c$. Furthermore, $X \Pi_x = X$, so the span of the right eigenvectors of $X$ must be $\mathcal{Y}_{\parallel}$. Since the span of the left eigenvectors of $T_i$ is a subspace of $\mathcal{Y}_{\parallel}$ for all $i$, the span of $M_\parallel$ is also a subspace of $\mathcal{Y}_{\parallel}$, and we can multiply $M_\parallel$ by $X$ on both sides without changing the rank and use Eq. \eqref{eq:Tdef2} to get a simpler $n_c\times n_c$ matrix
    \begin{align}
    \label{eq:Mtilde}
        \widetilde{M}_\parallel = X M_\parallel X^\top = \sum_{i=0}^{n_c-1}\frac{dX}{dV_i}(\mathbb{I}_N-\Pi_x) \frac{dX^\top}{dV_i}.
    \end{align}
    Since this has the same number of nonzero eigenvalues as $M_\parallel$, $m_\parallel$ is also given by the number of nonzero eigenvalues of $\widetilde{M}_\parallel$. This allow us to compute $m_\parallel$ without the need to calculate $X^{-1}$ directly like we would have if we calculated $M_\parallel$ using Eq. \eqref{eq:Tdef2} for $T_i$.
    \item[Case 2:] Let $\mathcal{Y}_{V}^{\perp}$ be the span of the set of vectors that contains $\hat{y}_{i,j}^{\perp}$ for every $i\in\{0,\dots,n_c-1\}$ and $j\in\{0,\dots,\mathrm{rank}(T_i)-1\}$, and let its dimension be $m_\perp$. Then if $Y$ is such that $Y^\top y^\perp = 0$ for all $y^\perp\in\mathcal{Y}_{V}^{\perp}$, then $\nabla_V S_{\mathrm{min}} = \mathbf{0}$ because $T_i Y = \mathbf{0}$ for all $i$. This includes the minimal case where $S_{\mathrm{min}}=0$, in which the RC perfectly describes $Y$ and no further improvement is possible. The dimension of the space of $Y$'s that fall under this case is $N-m_\perp$. This is because $\mathcal{Y}_{V}^{\perp}$ is spanned by $m_\perp$ basis vectors, so the space of $Y$'s that are orthogonal to all of them is spanned by the remaining $N-m_\perp$ basis vectors.  We can calculate $m_\perp$ from the matrix defined as
    \begin{align}
        M_\perp = \sum_{i=0}^{n_c-1} T_i^\top T_i.
    \end{align}
    $m_\perp$ is given by the rank of $M_\perp$ because each $T_i^\top T_i$ is a positive semi-definite matrix, so the only way that $M_\perp Y = \mathbf{0}$ is if $T_i Y = \mathbf{0}$ for all $i$. The dimension of the space of vectors that satisfy this relation is $N-m_\perp$ as mentioned previously, so if $M_\perp$ has $N-m_\perp$ zero eigenvalues, then that leaves $m_\perp$ nonzero eigenvalues.
    \item[Case 3:] Neither of the above cases holds, but every $Y^\top T_i Y = 0$ nonetheless. Since there are $n_c$ different $T_i$'s, then we have $n_c$ equations of constraint on which $Y$'s of this type set $\nabla_V S_{\mathrm{min}}$ to zero. However, it may be possible that some of these equations are not independent, which would imply that there is at least one linear combination of $T_i$'s such that $\sum_{i=0}^{n_c-1}\gamma_iT_i = \mathbf{0}$. Define $m_I$ to be the number of independent constraints of this form. Then the dimension of the space of $Y$'s for which $Y^\top T_i Y = 0$ is given by $N-m_I$, the number of free parameters left after applying the constraints. We can calculate $m_I$ using the $n_c\times n_c$ matrix $M_I$ whose components are defined to be
    \begin{align}
        (M_I)_{ij} = \mathrm{Tr}(T_i T_j^\top).
    \end{align}
    $m_I$ is given by the rank of $M_I$, since if the linear combination $\sum_{i=0}^{n_c-1}\gamma_i T_i = \mathbf{0}$ then the $n_c$-dimensional vector $v$ defined by $v_i = \gamma_i$ is an eigenvector of $M_I$ with eigenvalue $0$.
\end{itemize}
As long as dimensions of the spaces of $Y$'s that satisfy the three cases above are all smaller than the total dimension of all $Y$'s, then it is very unlikely that any given $Y$ will fall into any of these categories. The total dimension of the space of $Y$ vectors is $N$, and the dimension of the spaces for each of the three cases are $N-m_\parallel, N-m_\perp,$ and $N-m_I$, respectively, so as long as $m_\parallel, m_\perp,$ and $m_I$ are all positive, than this argument holds. We can check whether or not they are zero by taking the traces of their respective matrices $M_\parallel, M_\perp,$ and $M_I$ since they are all positive semi-definite matrices, so the only way that the sums of their eigenvalues are zero is if every eigenvalue is zero. It turns out that all three traces are equal, since 
\begin{align}
\label{eq:Trcond}
    \mathrm{Tr}(M_\parallel) = \mathrm{Tr}(M_\perp) = \mathrm{Tr}(M_I) = \sum_{i=0}^{n_c-1}\mathrm{Tr}(T_i T_i^\top),
\end{align}
so therefore if any one of $m_\parallel, m_\perp,$ or $m_I$ are found to be zero, then all three are guaranteed to be zero. This is because a matrix whose eigenvalues are all zero must be the zero matrix, so this implies that $T_i = \mathbf{0}$ for all $i$, which is the subject of our second proof below. Conversely, Eq. \eqref{eq:Trcond} also implies that if any one of them is positive, then they all must be positive as well, so we need only check that one of them is nonzero for this proof to hold. The easiest one to check is likely $m_\parallel$ as we can use $\widetilde{M}_\parallel$ in place of $M_\parallel$ and avoid directly calculating the $X^{-1}$.
\end{proof}

\begin{lemma}[Lower dimensionality of cases where $\nabla_V \Pi_x = \mathbf{0}$]
\label{lemmapt2}
    The space of ESNs defined by matrices, vectors, and input sequences $(A, B, \{u_k\})$ that satisfies $\nabla_V \Pi_x = \mathbf{0}$ is of lower dimension than the space of all possible ESNs.
\end{lemma}

\begin{proof}
Our second part of this proof will show that in the space of all matrices $A$, vectors $B$, and input sequences $\{u_k\}$ that lead to a convergent RC, the subspace of these where $\nabla_V S_{\mathrm{min}} = \mathbf{0}$ for all $Y$ (or equivalently that $T_i = \mathbf{0}$ for all $i$) must be of lower dimension as well. If it was not of lower dimension, then that would imply that $\nabla_V S_{\mathrm{min}} = \mathbf{0}$ holds over a finite region of the possible values of $A, B,$ and $\{u_k\}$. If the function $f(x_k,u_k)$ representing the reservoir dynamics is analytic, then all reservoir states must be analytic functions of $A, B,$ and $\{u_k\}$ as well. In this paper, we are considering ESNs with $f(x_k,u_k) = \sigma(A x_k + B u_k)$, where $\sigma(z)$ is the element-wise sigmoid function, which is analytic. So if all reservoir states are analytic in  $A, B,$ and $\{u_k\}$, then $S_{\mathrm{min}}$ is an analytic function of these variables as well, so if $\nabla_V S_{\mathrm{min}} = \mathbf{0}$ holds over a finite region of the possible values of these variables, then $S_{\mathrm{min}}$ must be constant with respect to $V$ for all $A, B, \{u_k\}$ and $\{y_k\}$, which is a direct consequence of the identity theorem of an analytical function \cite{Rudin:1976}. However, the ESN will become unstable for a $V$ with a sufficiently large norm, in which case the fit to $\{y_k\}$ will be poor and $S_{\mathrm{min}}$ would have to be larger than if we had no feedback. Thus $S_{\mathrm{min}}$ must not be constant with respect to $V$ for all $A, B, \{u_k\}$ and $\{y_k\}$, and therefore the space of matrices $A$, vectors $B$ and training inputs $\{u_k\}$ that satisfy $\nabla_V \Pi_x = \mathbf{0}$ is of lower dimension than the space of all possible $(A,B, \{u_k\})$.
\end{proof}

\begin{lemma}[Lower dimensionality of the subdomain of $S_{\mathrm{min}}(A+B  V^\top,B,\{u_k\}, \{y_k\})$ for which $\nabla_V S_{\mathrm{min}} = \mathbf{0}$]
\label{lemmafull}
    The dimension of the space of matrices, vectors, input sequences, and target sequences $(A,B,\{u_k\}, \{y_k\})$ that satisfy $\nabla_V S_{\mathrm{min}} = \mathbf{0}$ is strictly less than the dimension of the space of all possible $(A,B,\{u_k\}, \{y_k\})$, which implies that the number of cases in which $S_{\mathrm{min}}$ has a null gradient w.r.t. $V$ is vanishingly small compared to all cases.
\end{lemma}

\begin{proof}
Lemma \ref{lemmapt1} proves that when $\nabla_V \Pi_x \neq \mathbf{0}$, the number of cases where $\nabla_V S_{\mathrm{min}} = \mathbf{0}$ is vanishingly small in the space of all training sequences $\{y_k\}$, and therefore also in the space of all possible $(A,B,\{u_k\}, \{y_k\})$. Lemma \ref{lemmapt2} proves that when $\nabla_V \Pi_x = \mathbf{0}$, the number of cases where $\nabla_V S_{\mathrm{min}} = \mathbf{0}$ is vanishingly small in the space of all matrices $A$, vectors $B$, and training inputs $\{u_k\}$, and therefore for all training sequences  $\{y_k\}$ as well. Therefore, the number of cases of $(A,B,\{u_k\}, \{y_k\})$ where $\nabla_V S_{\mathrm{min}} = \mathbf{0}$ is vanishingly small over all possible $(A,B,\{u_k\}, \{y_k\})$, regardless of $\nabla_V \Pi_x$.
\end{proof}

Finally, we note that all of this work dedicated to finding where $\nabla_V S_{\mathrm{min}} = \mathbf{0}$ is a necessary but not sufficient condition for proving that feedback will not improve the result. That is, the points where $\nabla_V S_{\mathrm{min}} = \mathbf{0}$ correspond to the extrema of $S_{\mathrm{min}}$ with respect to $V$, but these extrema could be minima, maxima, or saddle points. However, only minima will prevent feedback from improving the output of an ESN, and if we use a non-local method to find the global minimum of $S_{\mathrm{min}}$ with respect to $V$, then local minima do not mitigate improvement, either.

It is possible to compute $\nabla_V S_{\mathrm{min}}$ without much extra overhead for any given run of a RC to see if it is zero. The derivatives $\frac{dx_k}{d V_i}$ can be calculated iteratively using the relation 
\begin{align}
    \frac{dx_k}{dV_i} &= \Sigma_k\left(B~x_{k-1,i}+\overline{A}\frac{dx_{k-1}}{dV_i}\right) \\
    \Sigma_{k,ij} &= \delta_{ij} \sigma'(z_{k-1,i})=\delta_{ij} x_{k,i}(1-x_{k,i}),
\end{align}
where $\overline{A} = A + B V^\top$. This uses $A$, $B$, and the reservoir states $x_k$ that have already been obtained from running of the RC. We can also avoid computing $N\times N$ matrices like $\Pi_x$ directly by noting that from previous results we have $Y \Pi_x Y = K_{xy} K_{xx}^{-1} K_{xy}$. $K_{xx}^{-1}$ was already computed when optimizing for $W$, so there is no additional matrix inversion needed to find $\nabla_V S_{\mathrm{min}}$. It is also feasible to check whether this due to the specific $Y$ or a symptom of the RC by checking if $m_\parallel = 0$ using $\mathrm{Tr}(\widetilde{M}_\parallel)$ defined in Eq. \eqref{eq:Mtilde}. $\widetilde{M}_\parallel$ has the same form as $S_{\mathrm{min}}$, but with $Y$ replaced with the matrix $\frac{dX}{dV_i}$, so by replacing $y_k$ with $\frac{dx_{k,j}}{d V_i}$ for every $i$ and $j$ in the definition of $K_{xy}$ we can check if $\mathrm{Tr}(\widetilde{M}_\parallel)$ is zero without much extra work.

\subsection{Proving the Universal Superiority of ESNs with Feedback}
We are now ready to finally present the proof of Theorem \ref{thm1} since we know that $\nabla_V S_{\mathrm{min}}$ is nonzero except on a lower dimensional subspace. 

\begin{proof}[\textbf{Proof of Theorem \ref{thm1}}]
Note that $S_\mathrm{min}$ is a real analytic functional with respect to matrix $A$ (see the definition of $S$ in Eq. \eqref{eq:cost-function}), having the Taylor series
\begin{align}
    &S_\mathrm{min} (A+B V^\top,B,\{u_k\}, \{y_k\}) \nonumber \\
    &= S_\mathrm{min} (A,B, \{u_k\}, \{y_k\}) + \mathrm{Tr} \left[ \left(\nabla_A S_\mathrm{min} (A,B, \{u_k\}, \{y_k\}) \right) (B V^\top)^\top \right] + \mathcal{O}[\delta A^2], \label{eq:Taylor-S2}
\end{align}
where the last term consolidates the second and the higher order of $\delta A = (A + B V^\top) - A = B V^\top$. A reasonable ansatz for $V$ that reduces the second term most is
\begin{equation}
    V = - \alpha \nabla_A S_\mathrm{min} (A,B,\{u_k\}, \{y_k\})^\top  B,
\end{equation}
where $\alpha>0$ is a constant to be determined. We now calculate the second term as
\begin{equation}
    \mathrm{Tr} [ \nabla_A S_\mathrm{min} (A,B,\{u_k\}, \{y_k\})) (B V^\top)^\top] = -\alpha \beta,
\end{equation}
where
\begin{align}
    \beta &= \mathrm{Tr} [ \nabla_A S_\mathrm{min} (A,B,\{u_k\}, \{y_k\})^\top (B B^\top) \nabla_A S_\mathrm{min} (A,B,\{u_k\}, \{y_k\}) ] \nonumber \\
    &= || \nabla_A S_\mathrm{min}(A,B,\{u_k\}, \{y_k\})^\top  B ||^2 \geq 0.
\end{align}
If $\beta > 0$, one can always choose an arbitrarily small $\alpha(>0)$ such that
\begin{equation}
    \alpha > \frac{\left| \mathcal{O} (\alpha^2) \right|}{\beta}. \label{eq:alpha-inequality}
\end{equation}
This is because the left side is linear with respect to $\alpha$ whereas the right side is higher-order polynomial of $\alpha$, and therefore, such (arbitrarily small) $\alpha$  satisfying the above always exists. 

The only case where such $\alpha$ cannot be found is the case where $\beta = 0$:
\begin{equation}
    \nabla_V S_{\mathrm{min}}(A+B  V^\top,B,\{u_k\}, \{y_k\}) = \nabla_A S_{\mathrm{min}}(A+B  V^\top,B,\{u_k\}, \{y_k\})^\top  B = \mathbf{0}.\label{eq:dead}
\end{equation}
However, we have proved in Lemma \ref{lemmafull} that the number of cases in which this occurs is vanishingly small. Then, the strict inequality in Eq. \eqref{eq:thm1-inequality} is proved in almost every case.  

Now, we prove that $\overline{A} = A + B  V^\top$ with $V = - \alpha \nabla_A S_\mathrm{min} (A,B,\{u_k\}, \{y_k\})^\top  B$ will make the ESN convergent. The set $\mathcal{A}_a =\{A \in \mathbb{R}^{n \times n} \mid  A^{\top}A < a^2 \mathbb{I}_n\}$ is an open convex set in $\mathbb{R}^{n \times n}$ for any $a>0$. As has been shown earlier, for the overwhelming majority of $(B, \{u_k\}, \{y_k\})$ there is always a choice of $V$ that decreases  $S_{\rm min}(A,B,\{u_k\}, \{y_k\})$. Since $A \in \mathcal{A}_a$, by the  continuity of the maximum singular value of $A$ with respect to $A$ (for any choice of matrix norm) and by the  particular choice of $V$, there always exists a small number $\delta>0$ such that $A + \delta B  V^{\top} \in \mathcal{A}_a$ (guaranteeing that the ESN with feedback remains convergent) while still decreasing the cost, $S_{\rm min}(A + B \delta V^{\top},B,\{u_k\}, \{y_k\})< S_{\rm min}(A,B,\{u_k\}, \{y_k\}$). This concludes the proof of Theorem \ref{thm1}.

\end{proof}

\subsection{Superiority of ESNs with Feedback for the Whole Class of ESNs}

According to Theorem \ref{thm1}, the cost of the ESN with feedback is guaranteed to be smaller than the cost of the ESN without feedback for almost all fixed $(A,B,\{u_k\}, \{y_k\})$. Next, the following corollary states that ESN with feedback exceeds the performance of ESN without feedback in the whole class of ESNs.

\begin{corollary}[Universal superiority of ESN with feedback over the whole class]
   For given and fixed finite input and output sequences $\{u_k\}=\{u_k\}_{k=1,\ldots,N}$ and $\{y_k\}=\{y_k\}_{k=1,\ldots,N}$, let $A$ and $B$ be drawn randomly according to some probability measure $\mathbb{P}$ on $\mathbb{R}^{n \times n} \times \mathbb{R}^n$. Let $\mathcal{X} = \{(A,B) \in \mathbb{R}^{n \times n} \times \mathbb{R}^n \mid A^{\top}A < a^2 \mathbb{I} \}$ and $\mathbb{P}$ be such that $\mathbb{P}(\mathcal{X})=1$. Let $\mathcal{Y} = \mathcal{X} \cap \{(A,B) \mid \nabla_{V}S_{\rm min}(A+BV^{\top},B,\{u_k\},\{y_k\}) \neq 0 \}$ and choose $\mathbb{P}$ such that $\mathbb{P}(\mathcal{Y})>0$. Let $\langle S_\mathrm{min} \rangle_{A, B} = \mathbb{E}\left[S_\mathrm{min}(A,B)\right]$, where the expectation (average) is taken with respect to the probability measure $\mathbb{P}$, for a fixed training dataset $\{u_k\}, \{y_k\}$. Then, for the given training data set, ESN with feedback on average has a smaller cost function values than ESN without feedback for various $(A,B)$ on average. That is, the following holds on average over all possible $(A,B)$:
    \begin{equation}
        \langle S_\mathrm{min} (A,B,\{u_k\}, \{y_k\}) \rangle_{A,B} > \langle \min_V S_\mathrm{min} (A + B V^\top, B,\{u_k\}, \{y_k\}) \rangle_{A,B}.
    \end{equation}
\end{corollary}

\begin{proof}
To prove the theorem over the whole class, we adopt a slightly different approach. The cost function after minimizing for $C$ and $W$ can be written as a function of the ESN parameters $S_{\mathrm{min}}(A,B)$ (short for $S_\mathrm{min} (A,B,\{u_k\}, \{y_k\})$). With feedback using a vector $V$, the new minimum is given by $S_{\mathrm{min}}(\overline{A},B) = S_{\mathrm{min}}(A + B V^\top,B)$. Let us separate the  matrix $A$ into two components given by
\begin{align}
    A &= A_{||}+A_\perp, \\
    A_{||} &= B  Z_{AB}^\top \quad \mathrm{where} \quad Z_{AB} = \frac{A^\top B}{||B||^2}, \\
    A_{\perp} &= \left(\mathbb{I}_{\mathrm{dim}(A)}-\frac{B B^\top}{||B||^2} \right) A = A-B Z_{AB}^\top,
\end{align}
This is essentially pulling the degrees of freedom of $A$ in the $B$ direction apart from all other degree of freedom, so that $A_{||}$ and $A_{\perp}$ are independent. This is best shown by calculating the inner product:
\begin{align}
    \mathrm{Tr} \left[ A_{||}^\top A_\perp \right] &= \mathrm{Tr} \left[ A^\top\frac{B B^\top}{||B||^2}\left(\mathbb{I}_{\mathrm{dim}(A)}-\frac{B B^\top}{||B||^2} \right) A \right] = \mathrm{Tr} \left[A^\top\left(\frac{B B^\top}{||B||^2}-\frac{B B^\top}{||B||^2} \right) A\right] = 0.
\end{align}
Thus we can write the minimized cost function as a function of these independent variables so that we can define $\tilde{S}_{\mathrm{min}}(Z_{AB},A_{\perp},B) = S_{\mathrm{min}}(A,B)$.
With feedback, these quantities become $\overline{Z}_{AB}= Z_{AB}+V$ and $\overline{A}_{\perp} = \overline{A}-B \overline{Z}_{AB}^\top = A-B Z_{AB}^\top = A_{\perp}$, and therefore $S_{\mathrm{min}}(\overline{A},B) = \tilde{S}_{\mathrm{min}}(Z_{AB}+V,A_{\perp},B)$. This all means that when we use feedback and optimize with respect to $V$ we are equivalently choosing $\min_{V}(\tilde{S}_{\mathrm{min}}(V,A_{\perp},B))$, the minimum value of $\tilde{S}_{\mathrm{min}}$ as a function of $Z_{AB}$, subject to the convergence constraint. We note that, therefore, 
\begin{equation}
    \min_{V}(\tilde{S}_{\mathrm{min}}(Z_{AB}+V,A_{\perp},B)) \leq \tilde{S}_{\mathrm{min}}(Z_{AB},A_{\perp},B), \label{eq:interim1}
\end{equation} 
since the cost function can further be reduced by optimizing an additional degree of freedom $A_{||}$ (equivalently, $V$). The equality occurs when $V = \mathbf{0}$ is the optimal solution, making the feedback unnecessary. One can verify that the derivative of the cost function with respect to $V$ is not zero except for vanishingly small cases of $(A,B)$. To show this, let us take the derivative of $S_{\mathrm{min}}$ at $V = \mathbf{0}$:
\begin{equation}
    \nabla_V S_\mathrm{min} (A + B V^\top, B)|_{V = \mathbf{0}} = B^\top  \nabla_A S_\mathrm{min} (A, B).
\end{equation}
The condition for this to become zero is exactly the same condition appearing in Eq. \eqref{eq:dead}, for which we proved that only vanishingly small number of $(A,B)$ satisfy the above equation. Therefore, the strict inequality holds for most of $(A,B)$. 

The average cost without feedback is given by $\langle \tilde{S}_{\mathrm{min}}\rangle_{Z_{AB},A_{\perp},B}=\mathbb{E}\left[  \tilde{S}_{\mathrm{min}}(Z_{AB},A_{\perp},B)\right]$, averaging over all three variables. But with feedback, as stated above this is equivalent to minimizing the cost with respect to $Z_{AB}$, so the average cost with feedback is given by $\langle \min_{V}(\tilde{S}_{\mathrm{min}}(V,A_{\perp},B))\rangle_{A_{\perp},B}$. The average over $Z_{AB}$ does nothing in this case because all of the different initial $Z_{AB}$ will get shifted to the minimizing value. Because $\tilde{S}_{\mathrm{min}}(Z_{AB},A_{\perp},B)\geq\min_{V}(\tilde{S}_{\mathrm{min}}(V,A_{\perp},B))$ for each individual choice of $Z_{AB}, A_{\perp},$ and $B$, and $Z_{AB}$ and $A_{\perp}$ are measurable functions of $A$ and $B$ by construction, we must have that $\langle \tilde{S}_{\mathrm{min}}(Z_{AB},A_{\perp},B)\rangle_{Z_{AB}} \triangleq \mathbb{E}[\tilde{S}_{\mathrm{min}}(Z_{AB},A_{\perp},B)\mid A^{\perp},B] \geq\min_{V}(\tilde{S}_{\mathrm{min}}(V,A_{\perp},B))$ for all $(A,B) \in \mathcal{X}$. Furthermore, equality only holds if every choice of $Z_{AB}$ yields the same value of $\tilde{S}_{\mathrm{min}}(Z_{AB},A_{\perp},B)$. By the definition of $\mathcal{X}$ and $\mathcal{Y}$ and the hypothesis that $\mathbb{P}(\mathcal{X})=1$, and since the number of $(A_\perp,B)$ that makes $\tilde{S}_\mathrm{min}(Z_{AB},A_\perp,B)$ completely independent from $Z_{AB}$ is vanishingly small (c.f. the proof of Theorem \ref{thm1}), one can always choose $\mathbb{P}$ under which $A$ and $B$ are sampled to be such that $\mathbb{P}(\mathcal{Y})>0$. Therefore, we have the strict inequality when averaged over $(A_\perp,B)$:
\begin{align}
    \lefteqn{\langle \tilde{S}_{\mathrm{min}}(Z_{AB},A_{\perp},B)\rangle_{Z_{AB},A_{\perp},B}} \notag \\
    &=\int_{\mathcal{Y}} \langle \tilde{S}_{\mathrm{min}}(Z_{AB},A_{\perp},B)\rangle_{Z_{AB}}(A,B) \mathbb{P}(dA,dB) \notag \\
    &\quad + \int_{\mathcal{X} \backslash \mathcal{Y}} \langle \tilde{S}_{\mathrm{min}}(Z_{AB},A_{\perp},B)\rangle_{Z_{AB}}(A,B) \mathbb{P}(dA,dB)\notag\\
&>     \int_{\mathcal{Y}} \min_{V} \tilde{S}_{\mathrm{min}}(V,A_{\perp},B) \mathbb{P}(dA,dB) + \int_{\mathcal{X} \backslash \mathcal{Y}} \min_{V} \tilde{S}_{\mathrm{min}}(V,A_{\perp},B) \mathbb{P}(dA,dB) \notag\\
&=\langle \min_V S_\mathrm{min} (A + B V^\top, B,\{u_k\}, \{y_k\}) \rangle_{A,B}.
\end{align}
Thus an ESN with feedback will always do better than an ESN without feedback over the whole ESN class on average, given the same number of computational nodes.
\end{proof}

We note that this corollary could be proven more succinctly using Theorem \ref{thm1} under the same hypothesis on the probability measure $\mathbb{P}$ under which $A$ and $B$ are sampled by arguing that the average over ESNs will always include cases where the strict inequality \eqref{eq:thm1-inequality} holds, and therefore the average also obeys a strict inequality. However, the proof given here adopted a different path from the proof of Theorem \ref{thm1}, providing an alternative explanation of the corollary.

\section{Optimization of ESN with Feedback}
\label{sec:optimize}

One of the main advantages to using an ESN is that the training procedure is a linear regression problem that can be solved exactly without much computational effort. To use the ESN to match a target sequence $\{y_k\}$ for a given input sequence $\{u_k\}$, we first run the ESN driven by the input for a number of steps until the initial state of the network is forgotten. This ensures that the states of the network is close to the unique sequence solely determined by the input that is guaranteed to exist by the uniform convergence property. In most of our simulations, we let the ESN run for $500$ steps before beginning training, which appears to be significantly more than necessary for our examples. We were able to use as few as $19$ steps of startup for some of our tests without any issue.

After this initial set of steps that insure that the system has converged to the input dependent state sequence, we then record the values of the state for the entire range of training steps, which we define to be a total of $N$ steps staring from $k=0$. We then define the network's output to be $\hat{y}_k = W^\top x_k+C$, where the parameters $C$ and $W$ are optimized using the cost function given in Eq. \eqref{eq:cost-function}, and whose exact solutions are given in Eqs. (\ref{eq:optimizeC},\ref{eq:optimize}). Then, any future step in $\{y_k\}$ is estimated using $\hat{y}_k = W^\top x_k+C$ for some $k\ge N$.

With feedback, we also must optimize with respect to the feedback vector $V$ to determine the modified input sequence $\{u_k+V^\top x_k\}$. Since the network states $\{x_k\}$ will have a highly complex and nonlinear dependence on $V$, we cannot solve it exactly as we do with $C$ and $W$. It also turns out that, unfortunately, the cost function is not a convex function with respect to $V$ (see Fig. \ref{fig:nonconvex}). Therefore, a simple gradient descent or a linear regression for optimizing (training) $V$ is not guaranteed to converge to the global minimum of the cost function. In this case, optimizing $V$ to minimize the cost function is tricky. 

\begin{figure}
\centering
\includegraphics[width=0.55\linewidth]{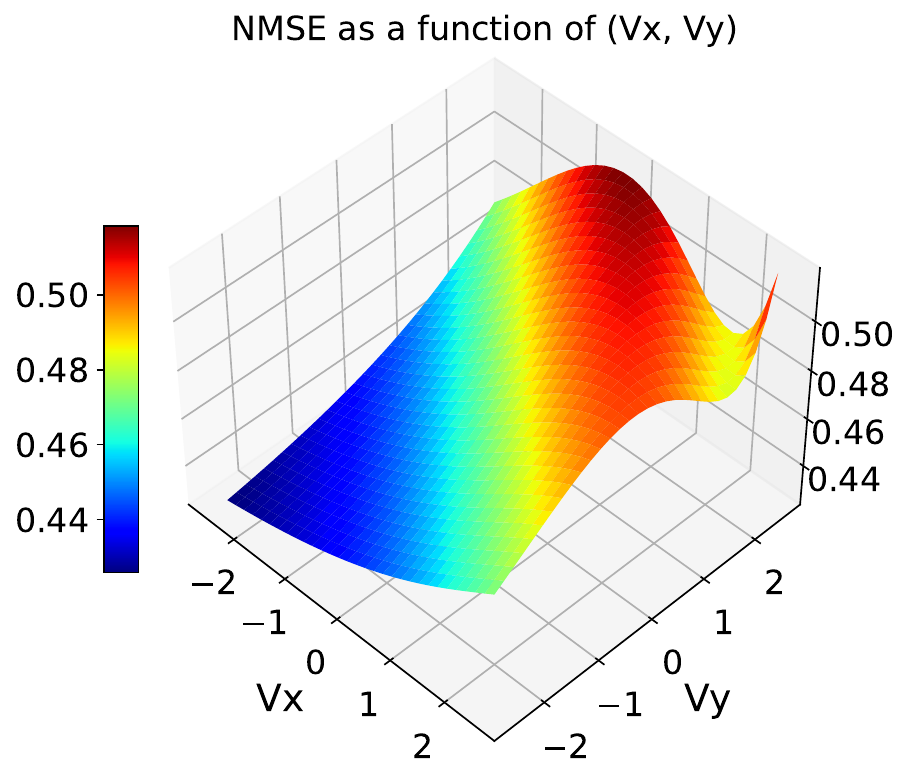}
\caption{3D plot of the non-convex dependence of the NMSE on $V$ for an ESN with 2 computational modes. Here we optimize $C$ and $W$ for the Mackey-Glass task and use 1000 training data points after 500 initial steps in our ESN. This plot shows only a portion of the full space of convergent feedback vectors to better illustrate the non-convexity.}
\label{fig:nonconvex}
\end{figure}

In fact, a good candidate for $V$ (which is at least locally optimal) is obtained by choosing $\alpha = \overline{\alpha}$ such that the difference between the left and the right hand side of the inequality \eqref{eq:alpha-inequality} becomes maximum: 
\begin{equation}
    \overline{\alpha} = \operatorname*{argmax}_\alpha  \left( \alpha - \frac{\left|\mathcal{O}(\alpha^2)\right|}{\beta} \right).
\end{equation}
Then, a good $V$ should be given by $V = - \overline{\alpha} \nabla_A S_\mathrm{min} (A,B)^\top  B$. Therefore, a strategy to optimize $V$ is, first, to perform the optimization for $W$ assuming there is no feedback, which will result in $S_\mathrm{min} (A,B)$. Then, one calculates $\nabla_A S_\mathrm{min} (A,B)$, which will require additional optimization of $W$ for given perturbed $A + \Delta A$ for each entry of $A$. Such perturbation requires $n \times n$ number of optimizations of $W$'s for different $A$'s. This will allow to calculate $\nabla_A S_\mathrm{min}(A,B)$, which will lead to a good $V = - \alpha \nabla_A S_\mathrm{min} (A,B)^\top B$. We, however, note that obtaining $\overline{\alpha}$ requires the calculations of $\nabla_A S_\mathrm{min} (A,B)$, $\nabla_A^2 S_\mathrm{min} (A,B)$, $\nabla_A^3 S_\mathrm{min} (A,B)$, etc., which are computationally demanding. Thus, in practice, we use different method that is more practical. 

In our numerical examples, we used a standard batch gradient descent method to optimize $V$, with a forced condition that ensures that the ESN will remain convergent during every step. We cannot use stochastic gradient descent because of the causal nature of the ESN, meaning that the order and size of the training set influences the optimal value of $V$. The proofs in Section \ref{sec:proof} guarantee that gradient descent will almost always provide an improvement to the fit. The mathematical details of the gradient descent method that we used is explained in \ref{sec: Details of Batch Gradient Descent Method used in Numerical Examples}. 

The method we detail in \ref{sec: Details of Batch Gradient Descent Method used in Numerical Examples} is essentially a reduced version of the standard real-time recurrent learning (RTRL) algorithm \cite{williams1989learning, robinson1987utility}, where only some of the internal parameters are modified corresponding to the vector $V$ rather than all of the parameters in $A$. For an ESN with $n_c$ computational nodes, the memory requirements and computational complexity per time step for our algorithm are $\mathcal{O}(n_c^2)$ and $\mathcal{O}(n_c^3)$, respectively, in contrast to the $\mathcal{O}(n_c^3)$ and $\mathcal{O}(n_c^4)$ complexities for full RTRL \cite{williams1995gradient}. This is because our method only requires the use of the rank 2 tensor $\frac{dx_{k,l}}{dV_i}$ at a time step $k$ rather than the rank 3 tensor $\frac{dx_{k,l}}{dA_{ij}}$ used in full RTRL. Alternatively, one could also use the standard backpropagation through time (BPTT) algorithm \cite{rumelhart1986learning, werbos1990backpropagation} to optimize for $V$. The derivative of the cost function $S_\mathrm{min}(A,B)$ with respect to a component of $V$ is given by
\begin{align}
    \frac{dS_\mathrm{min}(A,B)}{dV_i} = \sum_{j,l=1}^{n_c}\frac{dS_\mathrm{min}(A,B)}{d(A_{jl}+B_jV_l)}\frac{d(A_{jl}+B_jV_l)}{dV_i} = \sum_{j=1}^{n_c}\frac{dS_\mathrm{min}(A,B)}{dA_{ji}}B_j.
\end{align}
The standard BPTT algorithm can then be applied to calculate $\nabla_V S_\mathrm{min} (A,B) = \nabla_A^\top S_\mathrm{min} (A,B) B$ with a memory requirement of $\mathcal{O}(n_c N)$ and computational complexity per time step of $\mathcal{O}(n_c^2)$, where $N$ is the number of time steps in the training data set.

To begin our gradient descent routine, we start at $V_0=\mathbf{0}$. First, we run the ESN without feedback and optimize for $C$ and $W$ as usual. Then, we calculate $\nabla_V S_{\mathrm{min}}$, which as we will demonstrate below can be done using only quantities already obtained from running the ESN. We then choose $V_1 = V_0 - \eta \nabla_V S_{\mathrm{min}}$ to be our new feedback vector for the next step, where $\eta$ is a learning rate that must be chosen beforehand. We repeat this process many times, running the ESN with $\{u_k+V_i^\top x_k\}$ as the input sequence, optimizing for $C$ and $W$ under the new inputs, and then recalculating $\nabla_V S_{\mathrm{min}}$ to update the feedback vector for the next step using $V_{i+1} = V_i - \eta \nabla_V S_{\mathrm{min}}$. This is performed for a set number of iterations. In the event that the gradient descent converges to an ESN that is unstable, we will also detail a procedure we use to keep every $V_i$ within a certain convex region for which the ESN is guaranteed to be stable.

In a comparison between the computational complexities of training an ESN with feedback versus and ESN without, the main two considerations will be the number of gradient descent steps $n_{\mathrm{grad}}$ taken as well as the complexity of computing the gradients discussed above. Let us first define the computational complexity per time step of training an ESN with $n_c$ computational nodes without feedback to be $\mathcal{O}(\mathcal{C}_{n_c})$. For a software implementation, $\mathcal{C}_{n_c}$ will be at least $n_c^2$ due to the linear transformation $Ax_k+B u_k$ at each time step $k$, but it will likely be larger due to the nonlinear activation function. The value of $\mathcal{C}_{n_c}$ in a hardware implementation may be significantly lower than on software. In either case, applying feedback will change the computational cost per time step to $\mathcal{O}((n_{\mathrm{grad}}+1)\mathcal{C}_{n_c} + n_{\mathrm{grad}}n_c^2)$ using BPTT for the gradient calculations. The additional memory requirements will only be $\mathcal{O}(n_c N)$ for $N$ training data points coming from the gradient calculation since data from previous gradient descent runs can be discarded.

Here, we present the method to enforce the ESN's stability while updating $V$. For this, we need to make sure that the ESN remains convergent at every step of gradient descent. The constraint on $V$ is given by the constraint for convergence on the ESN following from Eq. \eqref{eq:constraint} with $\overline{A}$ in place of $A$. Formally, the constraint is $\overline{A}^\top \overline{A} < a^2 \mathbb{I}_{\mathrm{dim}(A)}$, where $a$ is a constant value that depends on the nonlinear function $g(z)$ of the ESN. We use the sigmoid function, so we take $a=4$. We ensure that the gradient descent algorithm obeys this constraint by applying a correction to any gradient descent step that causes the new value of $V$ to violate the constraint inequality. This correction is designed to only change the component of $V$ perpendicular to the surface defined by $\overline{A}^\top \overline{A} = a^2 \mathbb{I}_{\mathrm{dim}(A)}$ as a function of $V$, so that gradient descent can still freely adjust $V$ in any direction parallel to this surface.

For some small shift in $V$ given by $\delta V$, the change in $\overline{A}^\top\overline{A}$ is given by
\begin{align}
    \overline{A}^\top\overline{A} &= A^\top A + V (B^\top A) + (A^\top B) V^\top + ||B||^2~V V^\top \\ 
    \delta(\overline{A}^T\overline{A}) &\approx \delta V (B^\top A)+(A^\top B) \delta V^\top + ||B||^2~\delta V  V^\top+||B||^2~V  \delta V^\top \\
    &= \delta V (B^\top\overline{A})+(\overline{A}^\top B) \delta V^\top.
\end{align}
If gradient descent ends up causing the largest singular value $\lambda_{\max}$ of $A$ to reach or exceed $a$, it would cause our ESN to cease being uniformly convergent. We can use the associated normalized eigenvector $u_{\max}$ associated with $\lambda_{\max}$ to get
\begin{align}
    \delta(\lambda_{\max}^2) &= \delta(u_{\max}^\top\overline{A}^\top\overline{A} u_{\max}) \\
    &= u_{\max}^\top\delta(\overline{A}^\top\overline{A}) u_{\max}+2\delta(u_{\max})^\top\overline{A}^\top\overline{A} u_{\max} \\
    &= u_{\max}^\top\delta(\overline{A}^\top\overline{A}) u_{\max}+2\lambda_{\max}^2\delta V^\top\left(\frac{du_{\max}}{dV}\right)^\top u_{\max} \\
    &\approx 2(u_{\max}^\top\delta V)(B^\top\overline{A} u_{\max}) + 0.
\end{align}
We can ignore the dependence of $u_{\max}$ on $W_2$ in this equation because the derivative of a normalized vector is always orthogonal to the original vector, so the first-order shift in each $u_{\max}$ above gets eliminated by the other $u_{max}$. We can solve this in terms of $\delta V$ to get
\begin{align}
\label{eq:correction}
    u_{\max}^\top\delta V &\approx \frac{\delta(\lambda_{\max}^2)}{2(B^\top\overline{A} u_{\max})}.
\end{align}
This formula tells us that we can adjust the singular values of $A$ by using $\delta V' = \delta V - u_{\max}\frac{\Delta}{2(B^\top\overline{A} u_{\max})}$ for our gradient descent step instead of just $\delta V$ for some small positive value $\Delta$. We take $\Delta$ to be $\lambda_{\max}^2 + \delta(\lambda_{\max}^2) - a^2 + \epsilon_a$ for some small positive number $\epsilon_a$. This ensures that the new step $\delta V'$ will keep the singular values of $A$ strictly less than $a$ with a minimal change to the original step $\delta V$, so that the convergence of the gradient descent procedure is minimally impacted. $\lambda_{\max}^2 + \delta(\lambda_{\max}^2) - a^2$ is guaranteed to be of the same order of magnitude as the norm of $\delta V$ because we assume that $V$ leads to convergent dynamics, but $V + \delta V$ does not, so $\lambda_{\max}<a$ but $\lambda_{\max}^2 + \delta(\lambda_{\max}^2)\ge a^2$, and since $\delta(\lambda_{\max}^2)$ is of the same order as $||\delta V||$ the difference $0 \le \lambda_{\max}^2 + \delta(\lambda_{\max}^2) - a^2 < \delta(\lambda_{\max}^2)$ must be as well, where the last inequality is because $\lambda_{\max}<a$. If multiple singular values exceed $a$ due to a single step, we can apply this procedure for each singular value independently since the eigenvectors associated with these singular values are orthogonal, so each adjustment to $\delta V$ has no overlap with any of the other adjustments. We calculate $\lambda_{\max}^2 + \delta(\lambda_{\max}^2)$ and $u_{\max}$ directly from $(A + B(V + \delta V)^\top)^\top(A + B(V + \delta V)^\top)$, while $\epsilon_a$ is chosen to be $10^{-5}$. To the order of approximation used in Eq. \eqref{eq:correction}, we can take $A + B(V + \delta V)^\top\approx\overline{A}$ and use their eigenvectors interchangeably for our calculation of the adjustment to $\delta V$ aside from the value of $\delta(\lambda_{\max}^2)$.

\section{Benchmark Test Results}
\label{sec:tasks}

We conducted numerical demonstrations of ESNs with feedback by focusing on three distinct tasks: the Mackey-Glass task, the Nonlinear Channel Equalization task, and the Coupled Electric Drives task. These tasks are elaborated in the supplementary material of the reference \cite{Hulser:2023} for the first two and in \cite{Wigren:2017} for the latter. Each task represents a unique class of problems. The Mackey-Glass task exemplifies a highly nonlinear chaotic system, challenging the ESN's ability to handle complex dynamics. The Nonlinear Channel Equalization task involves the recovery of a discrete signal from a nonlinear channel, testing the ESN's proficiency in signal processing. Finally, the Coupled Electric Drives task is focused on system identification for a nonlinear stochastic system, evaluating the ESN's performance in modeling and memory retention. Together, these three tasks provide a comprehensive evaluation of ESNs, covering aspects like nonlinear modeling, system memory, and advanced signal processing. This multifaceted approach ensures a thorough assessment of ESN capabilities across various complex systems. We provide a comprehensive table of the simulation parameters we use for each task in Table \ref{tab1}, as well as a comprehensive table of the error measures for our simulations in Table \ref{tab2} toward the end of the section.

The Mackey-Glass task requires the ESN to approximate a chaotic dynamical system described by $y(t)$ in the Mackey-Glass equation:
\begin{align}
    \frac{dy}{dt}(t) = \beta\frac{y(t-\tau)}{1+y^n(t-\tau)}-\gamma y(t),
\end{align}
where we choose the standard values $\beta = 0.2, \tau = 17, n = 10,$ and $\gamma = 0.1$. We numerically approximate the solution to this equation using $y_{k+1} = y((k+1)\delta t) = y_k + \delta t \frac{dy_k}{dt}$ with $\delta t = 1.0$ and $y(0) = 1.0$. We also run the solution for $1000$ steps before using it for the task, or in other words the target sequence we use actually starts with $y_{1000}$. The task for the ESN is to predict what the sequence will be 10 time steps into the future using the past values of the target system. In other words, using the input sequence $\{u_k = y_{k-10}\}$, we want the ESN to successfully predict $\{y_k\}$. In our simulations for this task, we always run our ESNs for 500 time steps of startup, then train them over the next 1000 time steps, then use the following 500 time steps for our test set. We found that for this task, a learning rate of about 25.0 gives the best average performance boost using gradient descent for the feedback procedure.

The second task is the Nonlinear Channel Equalization task. In this task, there is some sequence of digits $\{d_k\}$, each of which can take one of 4 values so that $d_k\in\{-3,-1, 1, 3\}$, that is put through a nonlinear propagation channel. This channel $u_k$ is a polynomial in another linear channel $q_k$, which is in turn a linear combination of 10 different $d_k$ values. The linear channel is given by
\begin{align}
\nonumber q_k =&~0.08 d_{k+2} - 0.12 d_{k+1} + d_k + 0.18 d_{k-1} - 0.1 d_{k-2} \\ 
    &+ 0.091 d_{k-3} - 0.05 d_{k-4} + 0.04 d_{k-5} + 0.03 d_{k-6} + 0.01 d_{k-7}.
\end{align}
The nonlinear transformation $u_k$ of this channel is given by
\begin{align}
\label{eq:channel}
    u_k = q_k + 0.036 q_k^2 - 0.011 q_k^3 + v_k,
\end{align}
where $v_k$ is a Gaussian white noise term with a signal-to-noise ratio of $32\,\mathrm{dB}$. That is, each noise term $v_k$ is a random number generated from a Gaussian distribution with a mean of 0 and a standard deviation given by $\sigma_k = \mathrm{abs}(u_k) / 39.81$, so that the signal-to-noise ratio is $10 \log_{10}\left(\frac{u_k^2}{\sigma_k^2}\right) = 20 \log_{10}\left(39.81\right) \approx 32$. The task for the ESN is to recover the original digit sequence $\{d_k\}$ from the nonlinear channel $\{u_k\}$, which is used as input. In other words, using the input sequence $\{u_k\}$ described by Eq. \eqref{eq:channel}, we want the ESN to produce the digit sequence $\{y_k = d_k\}$ as output. Since the ESN produces a continuous output while the target sequence takes discrete values, we round the output of the ESN to the nearest value in $\{-3,-1, 1, 3\}$ for the final error analysis. However, we still use the continuous outputs during training using the standard cost function described in Eq. \eqref{eq:cost-function}. Just like for the Mackey-Glass task, we always run our ESNs for 500 time steps of startup, train them over the next 1000 time steps, then take the next 500 time steps as our test set. We found that for this task a learning rate of about 10.0 gives the best average performance boost with feedback.

The third test is fitting the Coupled Electric Drives data set, which is derived from a real physical process and is intended as a benchmark data set for nonlinear system identification \cite{Ljung99,Billings13}. In system identification, the ESN is used to approximately model an unknown stochastic dynamical system. This is achieved by tuning the free parameters of the ESN so that it approximates the nonlinear input/output (I/O) map generated by the unknown system
through I/O data generated by the latter. This I/O map sends input sequences $\{\ldots, u_{k-1},u_k\}$ to output sequences  $\{\ldots,y_{k-1},y_k\}$ for all $k$. To this end, the ESN is configured as a nonlinear stochastic autoregressive model following \cite{Chen:2022}, which is briefly explained in \ref{sec:nonlinear stochastic autoregressive model}. Because of the limited data for this task, we run our ESNs for just 19 time steps of startup, then train them over the next 280 time steps and use the following 200 time steps as the test set. We found that a learning rate of about 27.0 gives the best average performance with feedback.

We fit to the output signal labeled $z2$ in \cite{Wigren:2017}, which uses a PRBS input $u2$ with amplitude $1.5$. We model the data using a combination of the input $u2$ and the previous state of the system such that for each time step $k$, the next time step is obtained using an input given by $s \cdot u2_{k} + (1-s) \cdot  y_{k}$ for some parameter $s$. This parameter is chosen by optimizing the cost function of the training data using gradient descent, in much the same way that we optimize the feedback vector $V$. However, this procedure is used to provide information about both $u2$ and the past values of $y$ to the ESN through a single input channel, and is in no way related to our feedback procedure. In the context of the discussion in \ref{sec:nonlinear stochastic autoregressive model}, the function $\nu$ that is defined in that appendix is in this case given by $\nu(x_k,u2_k,y_k)= s \cdot u2_k + (1-s) \cdot y_k + V^{\top}x_k$, where $V=0$ if no state feedback is used, otherwise with feedback the value of $V$ is determined though the I/O data. Through our empirical analysis, we find that the global minimum of the cost function with respect to $s$ is always located near $s=0$, such that the cost function is locally convex in an interval that always contains $s=0$. Thus using gradient descent starting from $s=0$ to find the optimal value of $s$ will always converge to the global minimum of this parameter. In our numerical work below, we choose a learning rate of $0.0012$ without feedback and a learning rate of $0.001$ when using feedback.

\begin{table}
\centering
  \begin{tabular}{|c|c|c|c|c|c|c|}
    \hline
    \multirow{2}{*}{Figure} & \multirow{2}{*}{Training data} & \multirow{2}{*}{Test data} & \# of ESNs & \# of ESNs & Learning rate \\
    & & & (no feedback) & (w/ feedback) & (feedback only) \\ \hline
    \ref{fig:MGmain} & 1000 & 500 & 48000 & 9600 & 25.0 \\ \hline
    \ref{fig:CEmain} & 1000 & 500 & 48000 & 9600 & 10.0 \\ \hline
    \ref{fig:nodes} & 1000 & 500 & 9600 & 9600 & 25.0/10.0 \\ \hline
    \ref{fig:gradsteps} & 1000 & --- & --- & 9600 & 25.0/10.0 \\ \hline
    \ref{fig:CED} & 280 & 200 & 1 & 1 & 27.0 \\ \hline
  \end{tabular}
  \caption{Table of simulation parameters that vary between our figures. Any other model parameters (such as ESN hyperparameters) are either randomized or fixed for all simulations.}
\label{tab1}
\end{table}

In all of the following simulations, we randomly generated the internal parameters of the ESNs under the following rules. For the $A$ matrix, we randomly generated each element of the matrix from a uniform distribution on the interval $[-1,1]$. The interval [-1,1] was chosen because it seems to be generally effective for ESNs using the logistic activation function, and was not optimized for any specific task. In the event that the spectral radius of the matrix is greater than or equal to 4 (see Eq. \eqref{eq:constraint}), we rescale the entire matrix so that the spectral radius is randomly chosen from a uniform distribution on the interval [2,4). For the vector $B$, we similarly generated each element from a uniform distribution on the interval $[-1,1]$. We also employ Tikhonov regularization with a regularization factor of $10^{-10}$. Because we weight the cost function by the number of training data points as in Eq. \eqref{eq:cost-function}, this is equivalent to a value of $10^{-7}$ ($2.8 * 10^{-8}$ for the Coupled Electric Drives task) in other works that do not weight the cost function. When we apply gradient descent, we optimize the learning parameter to produce the best average performance boost on training data. We used 100 steps of gradient descent in all of our feedback simulations, except in Figure \ref{fig:gradsteps} where we analyze the performance of gradient descent.

\begin{figure}
\centering
\includegraphics[width=0.4\linewidth]{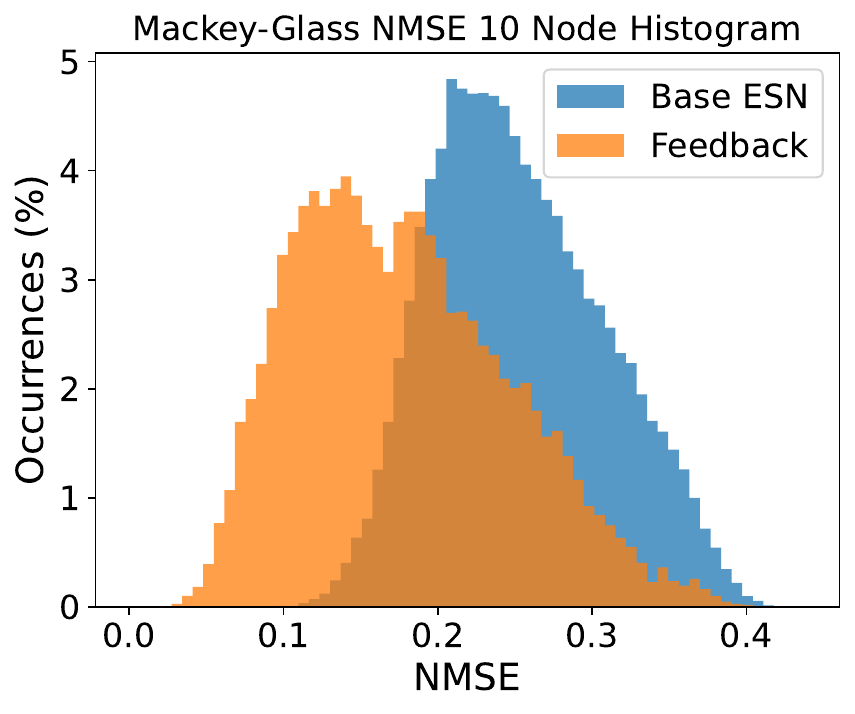}
\includegraphics[width=0.4\linewidth]{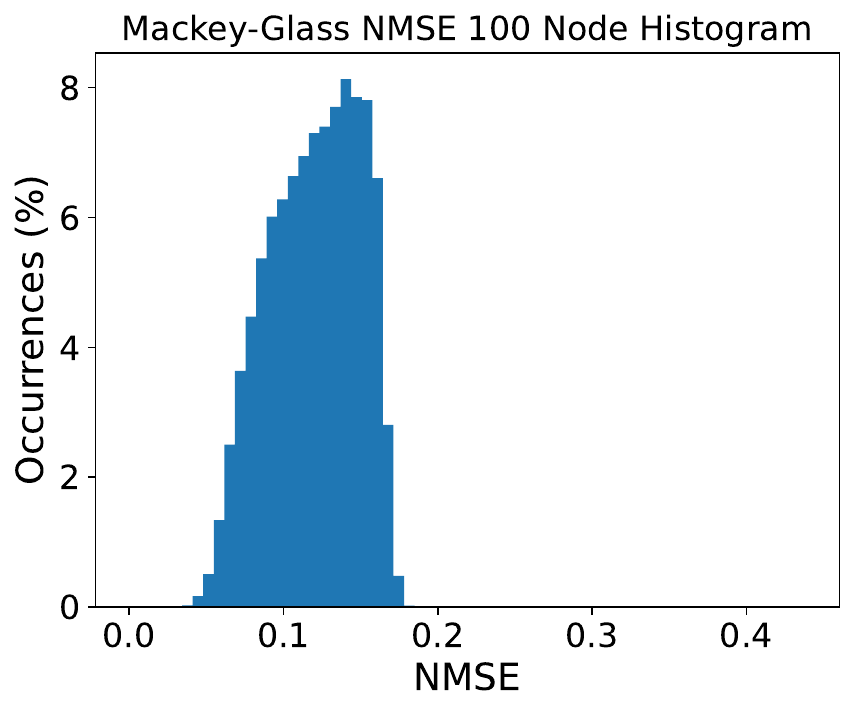}
\caption{Histograms for the NMSE values for the Mackey-Glass task. The left plot shows the NMSE values for ESNs with 10 computational nodes, with 48000 randomly chosen ESNs without feedback (the base ESN) and 9600 choices with feedback. For the feedback optimization, we used 100 steps of batch gradient descent with a learning rate of $25.0$. The right plot uses 48000 randomly chosen ESNs with 100 nodes. In all cases we used 1000 training data points taken after 500 steps of startup, and we show the NMSE values for 500 test data steps after training ends.}
\label{fig:MGmain}
\end{figure}

\subsection{Results on the Mackey-Glass task}
Fig. \ref{fig:MGmain} shows histograms of the NMSE values obtained during the Mackey-Glass task for many different ESNs. Specifically, we calculated the NMSE values for 48000 randomly chosen 10-node ESNs without feedback, another 9600 randomly chosen 10-nodes ESNs with feedback, and 48000 randomly chosen 100-node ESNs without feedback on the 500 test data points. We see that for 10 computational nodes without feedback, the distribution of NMSE values roughly takes the shape of a skewed Gaussian, with an average of about $0.252$ and a standard deviation of about $0.056$, and a longer tail on the right side than the left. With feedback, the distribution shifts significantly toward lower NMSE values, with an average of about $0.177$ and a standard deviation of about $0.069$. This is a roughly $30\%$ reduction of the average NMSE. For 100 computational nodes without feedback, the average is about $0.120$ with a standard deviation of about $0.029$, with a slight skew toward larger NMSE values this time. Note that the 10-node histogram with feedback appears to have two primary peaks, one centered around the lower edge of the 10-node distribution without feedback and one centered much closer to the 100-node average. With a better method of optimizing $V$, it may be possible to get more cases toward the left peak in this distribution and demonstrate results comparable to a 100-node calculation using only 10 nodes with feedback. Such an analysis is beyond the scope of this article, however.

We can also analyze how the NMSE on the 500 test data points works as a model hyperparameter selection criterion between many ESNs with and without feedback. Using the data from Fig. \ref{fig:MGmain}, we can check to see which ESN has the smallest NMSE out of all cases and see whether an ESN with feedback is selected as the best model of the true system. For the Mackey-Glass task, the smallest NMSE value among the 10-node ESNs without feedback is about $0.098$, while for the 10-node ESNs with feedback it is about $0.030$. Remarkably, the minimum NMSE among the 100-node ESNs without feedback is about $0.034$, so even if we include the 100-node ESNs in our set of possible models, the one that is selected for using the minimum NMSE criterion is a 10-node ESN with feedback.

\begin{figure}
\centering
\includegraphics[width=0.4\linewidth]{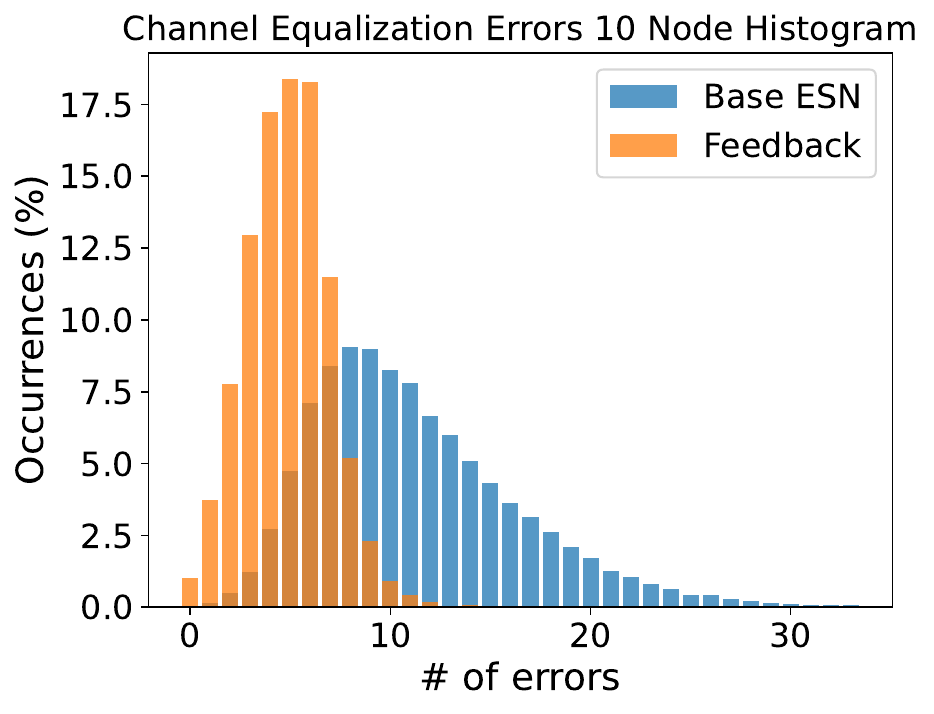}
\includegraphics[width=0.4\linewidth]{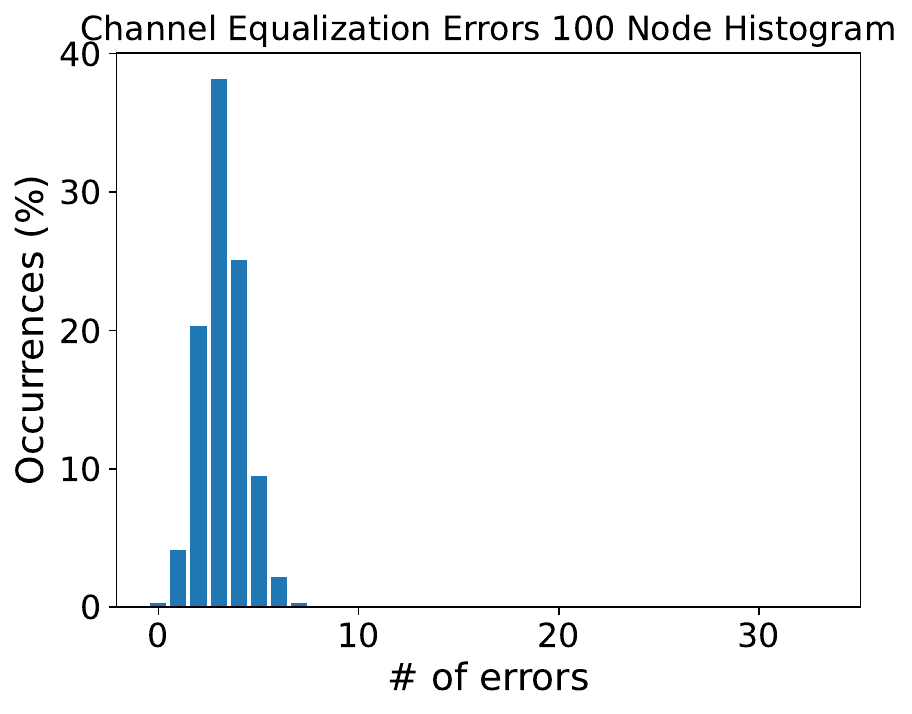}
\caption{Histograms for the errors in the Channel Equalization task. The left plot shows the number of errors for ESNs with 10 computational nodes, with 48000 randomly chosen ESNs without feedback (the base ESN) and 9600 choices with feedback. For the feedback optimization, we used 100 steps of batch gradient descent with a learning rate of $10.0$. The right plot uses 48000 randomly chosen ESNs with 100 nodes. In all cases we used 1000 training data points taken after 500 steps of startup, and we show the number of errors for 500 test data steps after training ends. The errors are counted using $\mathrm{abs}(d_k-\hat{y}_k) / 2$, where $d_k$ is the actual value of the signal at time step $k$, while $\hat{y}_k$ is the prediction by the ESN.}
\label{fig:CEmain}
\end{figure}

\subsection{Results on the Channel Equalization Task}

In Fig. \ref{fig:CEmain}, we show histograms of the number of errors obtained in the Channel Equalization task for many different ESNs. Specifically, we calculated the total number of errors made for 48000 randomly chosen 10-node ESNs without feedback, another 9600 randomly chosen 10-nodes ESNs with feedback, and 48000 randomly chosen 100-node ESNs without feedback. We count the number of errors based on how far off the ESN prediction is from the actual signal, using the expression $|d_k-\hat{y}_k| / 2$. For example, if the true signal value was $d_k=-1$ but the ESN gave us $\hat{y}_k=-3$, this counts as 1 error, but if $d_k=-1$ and $\hat{y}_k=+3$ we count it as 2 errors. We see that for 10 computational nodes without feedback, the distribution of total error values also takes the shape of a skewed Gaussian, with an average of about $11.29$ errors and a standard deviation of about $5.20$, and a longer tail on the right side than the left. With feedback, the distribution again shifts significantly toward a lower number of errors, with an average of about $4.89$ errors and a standard deviation of about $2.08$, a nearly $57\%$ reduction to the average error count. For 100 computational nodes without feedback, the average is about $3.22$ errors with a standard deviation of about $1.11$.

In this task, all of the distributions roughly adhere to the shape of a Gaussian with a long tail on the right. This is in contrast to the Mackey-Glass task, where each distribution had a slightly different skew. The feedback procedure has definitely improved the average performance of 10-node ESNs, but unlike the Mackey-Glass task there is no secondary peak, so it may be the case that in most of the cases the gradient descent algorithm has settled near a minimum and will not improve the results further. Still, the 10-node average with feedback is much closer to the 100-node result on average than the 10-node result.

Using the data from Fig. \ref{fig:CEmain}, we can perform a model hyperparameter selection based on which ESN makes the fewest errors out of all cases to see whether an ESN with feedback provides the best model of the true system. In this task, all 3 samples of ESNs contain models that make no errors on the test set. However, 99 of the 9600 10-node ESNs with feedback make no errors, but only 11 of the 48000 ESNs without feedback perform similarly. Accounting for the difference in sample size, this means that adding feedback made it roughly 45 times more likely to randomly generate an ESN with perfect performance on the test set. Among the 100-node ESNs without feedback, there were 154 out of 48000 ESNs that made no errors, which indicates that after adjusting for sample size a random 10-node ESN is still more likely to perform perfectly than a random 100-node ESN without. In terms of model hyperparameter selection, after adjusting for the smaller sample size of the 10-node ESNs with feedback, the type of ESN that would be most likely to be randomly selected among the models that make no errors would be a 10-node ESN with feedback.

\begin{figure}
\centering
\includegraphics[width=0.4\linewidth]{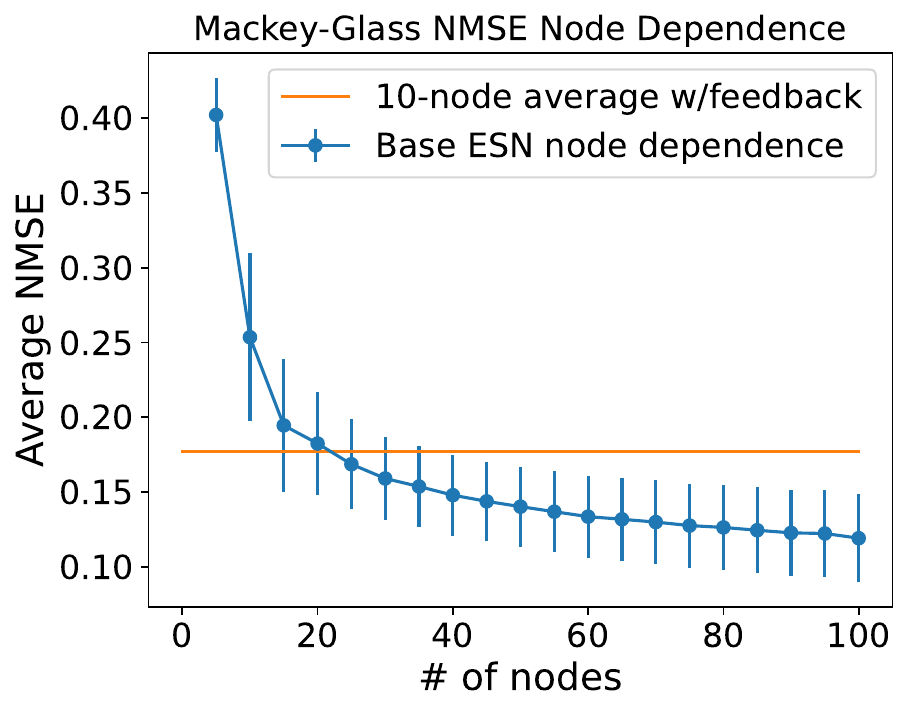}
\includegraphics[width=0.4\linewidth]{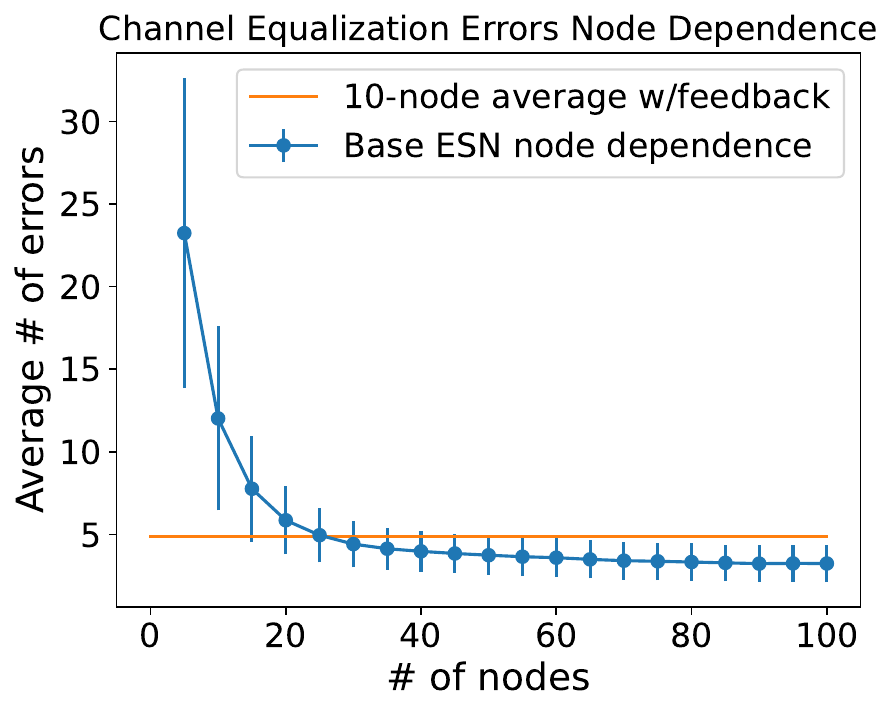}
\caption{Plots of the average NMSE and error values for the Mackey-Glass and Channel Equalization tasks, respectively, as a function of computational nodes. The average is taken over 9600 randomly chosen ESNs, and the error bars represent the standard deviation of the distribution for each number of nodes. In all cases we used 1000 training data points taken after 500 steps of startup, and we show the NMSE and total error values for 500 test data steps after training ends. We also include a line showing the average NMSE and total errors for 10 nodes with feedback for reference.}
\label{fig:nodes}
\end{figure}

\subsection{Node Dependence}

Fig. \ref{fig:nodes} shows the average dependence of the error of the ESN without feedback as a function of the number of nodes, averaging over 9600 randomly chosen ESNs for each data point, using the NMSE for the Mackey-Glass task and the total number of errors in the Channel Equalization task. We also include the average value of the 10-node results with feedback for comparison. The error bars in each plot represent the standard deviation associated with each specific number of nodes. We see that, on average, using feedback on a 10-node ESN with 100 steps of gradient descent for optimizing $V$ is roughly equivalent to a little more than a $20$ node calculation for the Mackey-Glass task, while it is closer to a $25$ or $30$ node calculation for Channel Equalization.

The main reason for this discrepancy is because the Mackey-Glass task is significantly more difficult for the ESN than the Channel Equalization task. In Mackey-Glass, we are asking the network to predict $10$ time steps into the future, but not all ESNs have a memory capacity going back $10$ steps, especially with only $10$ computational nodes. In contrast, the Channel Equalization task is much easier because the ESN does not have to reproduce the exact digits $\{-3,-1, 1, 3\}$, it only has to achieve a difference of less than 1 to be considered correct, so there is more room for error on with the continuous output of the ESN. This is evidenced by the fact that the nodal dependence in the Mackey-Glass task seems to continue decreasing almost linearly near 100 nodes, while for Channel Equalization the dependence is close to zero as the ESN is almost perfectly reproducing the signal with 100 nodes. This also suggests that the results for the Mackey-Glass task could see further improvement with feedback using a better method for optimizing $V$, since even the addition of computational nodes seems to converge slowly.

\begin{figure}
\centering
\includegraphics[width=0.4\linewidth]{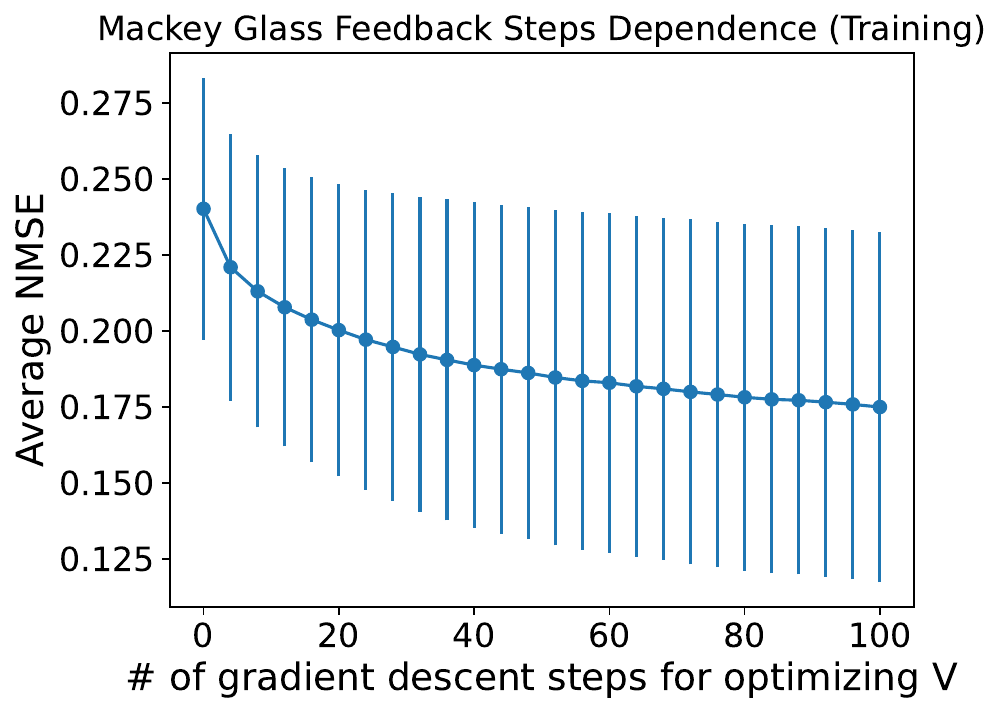}
\includegraphics[width=0.46\linewidth]{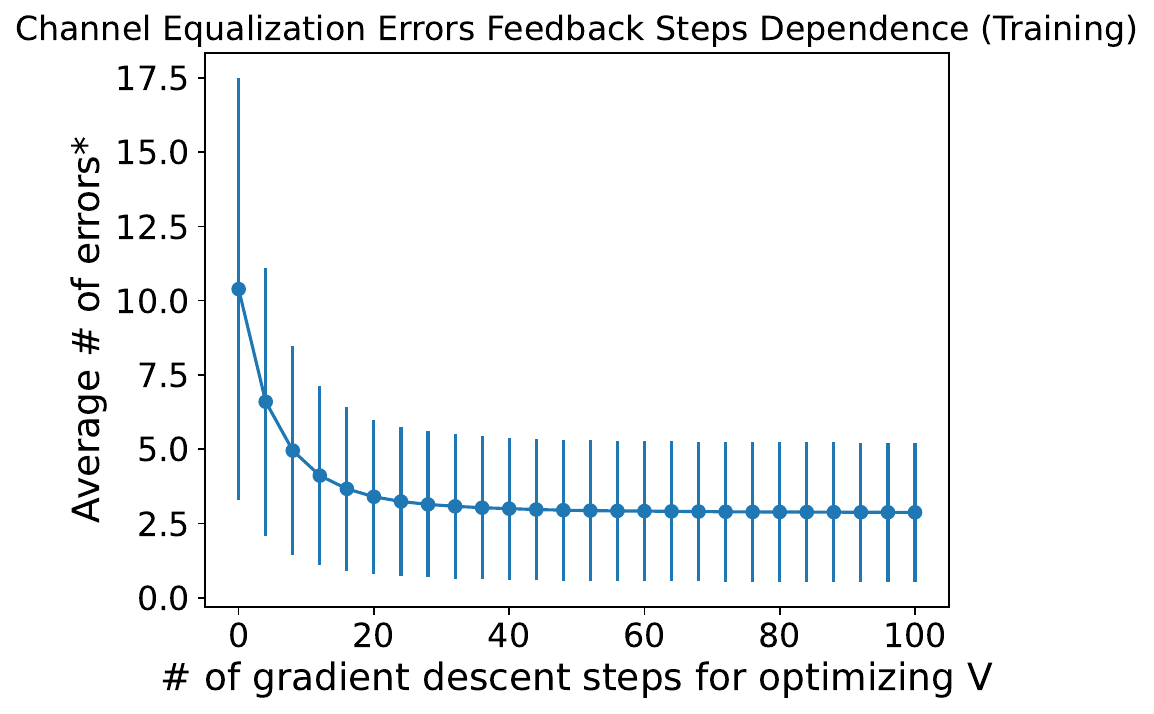}
\caption{Plots of the average NMSE and error values for the Mackey-Glass and Channel Equalization tasks, respectively, as a function of batch gradient descent steps for optimizing $V$, with a learning rate of $25.0$ for Mackey-Glass and $10.0$ for Channel Equalization. The average is taken over 9600 randomly chosen ESNs, and the error bars represent the standard deviation of the distribution for each number of nodes. In all cases we used 1000 training data points taken after 500 steps of startup. These plots shows the number of errors for the training data. The asterisk in the second plot indicates that we have rescaled the average number of errors by a factor of $1/2$ to be directly comparable with the other figures in this work.}
\label{fig:gradsteps}
\end{figure}

\subsection{Gradient Descent Step Dependence}

In Fig. \ref{fig:gradsteps}, we show the average dependence of the error of 10-node ESNs with feedback as a function of the number of gradient descent steps for optimizing $V$, averaging over 9600 randomly chosen ESNs for each data point, using the NMSE for the Mackey-Glass task and the total number of errors in the Channel Equalization task. Note that these plots show the NMSE and total errors for the training data set of $1000$ steps, as opposed to all of the previous plots which use the $500$ times immediate after training. This is why we have rescaled the average number of errors by a factor of $1/2$ in the plot for the Channel Equalization task: we are checking the errors of 1000 steps for each ESN in these plots instead of 500 like all the others, so the total number of errors is doubled as a result, hence the rescaling. The error bars in each plot represent the standard deviation associated with each specific number of gradient descent steps.

Here, we observe that for the Mackey-Glass task the gradient descent algorithm still has not fully converged even after 100 steps, and the variance in the performance is very large. In contrast, the gradient descent algorithm for the Channel Equalization task appears to have converged on average after about $25$ to $30$ steps to a value of about $3$, with a moderately large standard deviation. This corroborates the discussion of Fig. \ref{fig:nodes}, where we see that the ESN has trouble with the Mackey-Glass task, and so the convergence is slow and highly dependent on the specific ESN. Meanwhile, the Channel Equalization task is easier, and so the convergence occurs faster and more consistently. This further motivates using a better optimization method for $V$ to get the most we can out of an ESN for task like Mackey-Glass. For easier tasks like Channel Equalization we have likely already done the best we can with batch gradient descent, which seems to indicate that feedback can be more useful that adding an equal number of parameters as new computational nodes, at least for small ESNs.

\subsection{Results on the Coupled Electric Drives task}
\begin{figure}
\centering
\includegraphics[width=0.45\linewidth]{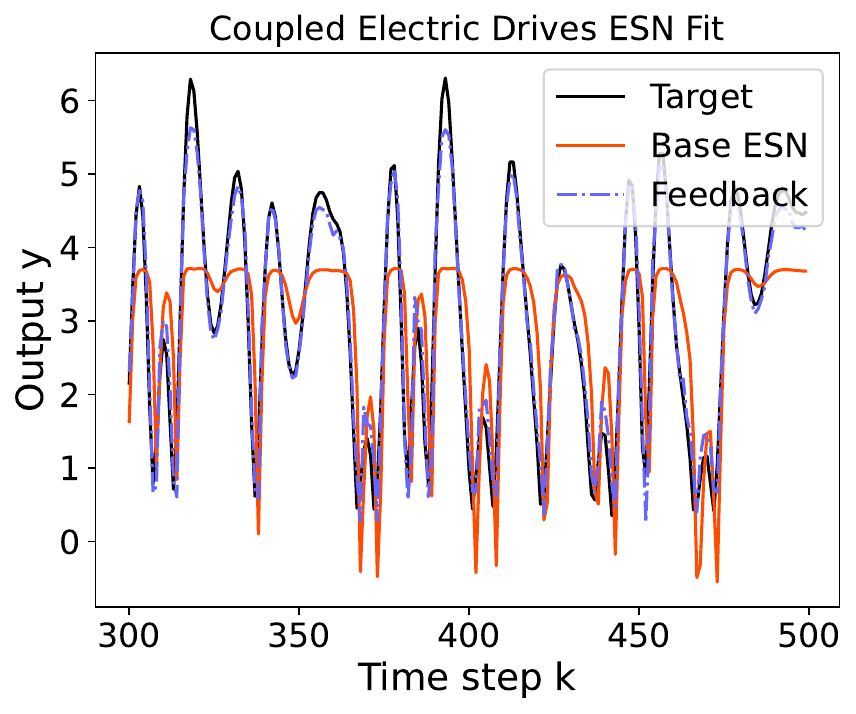}
\includegraphics[width=0.5\linewidth]{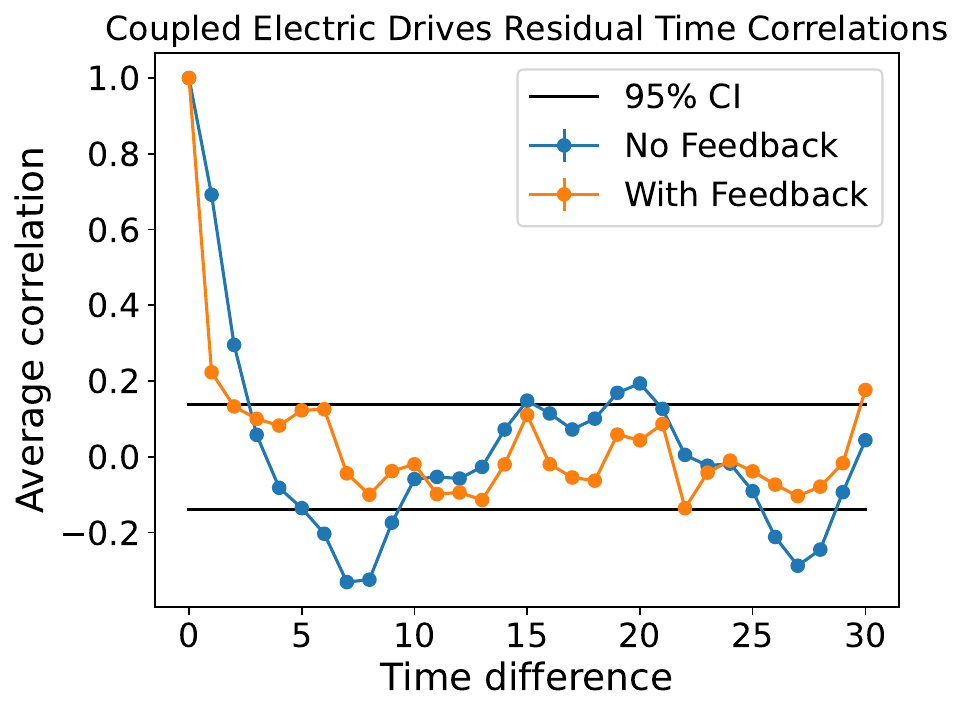}
\caption{Plots of the performance of a specific ESN with and without feedback for the Coupled Electric Drives task. In both cases we used an ESN with 2 computational nodes, with 280 training data points taken after 19 steps of startup, and we show the results for 200 test data steps after training ends. For the feedback, we used 100 steps of batch gradient descent with a learning rate of $27.0$. The left plot shows the ESN outputs with and without feedback (the base ESN) for the target data set along with the target data. In the right plot, we show the average correlation between time steps of the residuals $e_k = y_k - \hat{y}_k$, where the time difference denotes the difference between the time steps of the residuals involved in the average. The $95\%$ confidence interval denotes that there is a $95\%$ chance that an i.i.d. normal distribution would produce a correlation within this interval, indicating that data found largely within the interval are consistent with a normal distribution.}
\label{fig:CED}
\end{figure}

In Fig. \ref{fig:CED}, we show plots of one ESN's fit to the Coupled Electric Drive data given in \cite{Wigren:2017}. We specifically use $z2$, the PRBS signal with an amplitude of $1.5$. Note that since this data set contains only 500 data points, we use significantly less training data and test data than before, with only $280$ and $200$ steps, respectively. We are also only using 2 computational nodes here as opposed to $10$ or more in the previous tasks. This is in accordance with \cite{Chen:2022}, where it was shown that ESNs with 2 computational nodes perform well on the $z3$ Coupled Electric Drives data set. The first plot shows the fit to test data for a specific choice of of ESN with and without feedback. We see that in this specific instance the original ESN has some trouble fitting to the data, but with feedback the fit becomes much better, nearly matching the target data. The NMSE value without feedback was about $0.43$, but with feedback it is reduced by a full order of magnitude to about $0.032$.

The right plot shows the average correlations between the residuals $e_k = y_k - \hat{y}_k$ of the test data for a fixed time difference. This is calculated using the formula
\begin{align}
    R_{k} &= \frac{1}{N-1}\sum_{j=0}^{N-k-1}\frac{(e_j-\overline{e})(e_{j+k}-\overline{e})}{\sigma_e^2},
\end{align}
where $\overline{e}$ and $\sigma_e^2$ are the sample mean and variance of the residuals. This measures the correlation between the residuals at different time steps. If the residuals correspond to white noise, then $95\%$ of the correlations would fit into the confidence interval shown in the plot. We see that in the original ESN, there is a strong correlation between residuals that are one and two time steps apart, along with anti-correlations when they are 5 to 10 steps apart as well as from 25 to 30 step apart. With feedback, the one step correlation is significantly reduced, and all but one of the other correlations are found within the confidence interval. This suggests that the ESN without feedback was not completely capturing part of the correlations in the test data, but with feedback the accuracy is improved to the point that most of the statistically significant correlations have been eliminated. We also checked that the residuals are consistent with Gaussian noise using the Lilliefors test \cite{Lilliefors:1967, Abdi:2007} and checking a Q-Q plot \cite{Wilk:1968} against the CDF of a normal distribution. We found that both with and without feedback, the residuals pass the $n=50$ Lilliefors test and follow a roughly linear trend on the Q-Q plot. Finally, we checked the correlation between the residuals and the input $u2$ to see whether the residual noise is uncorrelated with the input and, therefore, unrelated to the system dynamics. We found that for this ESN, the residuals without feedback show a statistically significant anti-correlation of about $-0.22$, below the $95\%$ confidence threshold of $-0.14$. With feedback, the correlation becomes about $-0.11$, a significant reduction in magnitude that now puts it within the confidence interval.

The average behavior of feedback for this task also shows good improvement. Taking an average over 9600 different ESNs (including the one in Fig. \ref{fig:CED}), we find that the average NMSE without feedback using 2 computational nodes is about $0.178$ with a standard deviation of about $0.0779$, but with feedback this average goes down to $0.0723$ with a standard deviation of about $0.0448$. This is a roughly $59\%$ reduction of the average NMSE. Applying the NMSE on the 200 test data points as a model hyperparameter selection criterion for this task, we find that an ESN with feedback is selected for with an NMSE of about $0.024$, compared to a minimum of about $0.036$ without feedback. We also find that the average correlation with the input $u2$ is about $-0.183$ without feedback, below the $95\%$ confidence threshold of $-0.139$, but with feedback it is on average above the threshold with a value of about $-0.132$. Additionally, the average one time step difference correlation between residuals is about $0.69$ without feedback, but with feedback it reduces to about $0.48$. However, this does not reduce it into the $95\%$ confidence interval for Gaussian white noise, suggesting that there is still room for improvement with feedback. We note that batch gradient descent may not be the most effective choice for optimizing $V$ for this task, especially considering the use of a large learning rate of $27.0$ to get these results. If we were able to reach the true global minimum of $V$, we may be able to much more consistently reach a scenario like the one shown in Fig. \ref{fig:CED}.

\begin{table}
  \centering
  \begin{tabular}{|c|c|c|c|c|c|c|}
    \hline
    Task & \multicolumn{2}{|c|}{Mackey-Glass} & \multicolumn{2}{|c|}{Channel Equalization} & Electric Drives \\ \hline
    Nodes & 10 & 100 & 10 & 100 & 2 \\ \hline
    No Feedback (avg) & 0.252 & 0.120 & 11.29 & 3.22 & 0.178 \\ \hline
    Feedback (avg) & 0.177 & --- & 4.89 & --- & 0.072 \\ \hline
    No Feedback (best) & 0.098 & 0.034 & 0 (0.02\%) & 0 (0.32\%) & 0.036 \\ \hline
    Feedback (best) & 0.030 & --- & 0 (1.03\%) & --- & 0.024 \\ \hline
  \end{tabular}
  \caption{Table of values for the error measures in each of the tasks, organized by number of nodes, whether feedback was applied, and whether the average over many ESNs or the best ESN of the set is chosen. For the Mackey-Glass and Coupled Electric Drives tasks, the NMSE is used for the error measure, while for the Channel Equalization task we use the total number of errors as the error measure. For the best ESN in the Channel Equalization task, since each category has ESNs that made no errors on the test set, we also give the percent of ESNs that made no errors within the total number of ESNs in that category.}
  \label{tab2}
\end{table}

\section{Discussions and Conclusion}
\label{sec:conclude}

In this work, we have introduced a new method for improving the performance of an ESN using a feedback scheme. This scheme uses a combination of the existing input channel of the ESN and the previously measured values of the network state as a new input, so that no direct modification of the ESN is necessary. We proved rigorously that using feedback is almost always guaranteed to provide an improvement to the performance, and that the number of cases in which such an improvement is not possible using batch gradient descent is vanishingly small compared to all possible ESNs. In addition, we proved rigorously that such a feedback scheme provides a superior performance over the whole class of ESNs on average. We laid out the procedure for optimizing the ESN as a function of the fitting parameters $C, W,$ and $V$, exactly solving for $C$ and $W$ while using batch gradient descent on $V$ for a fixed number of steps. We then demonstrated the performance improvements for the Mackey-Glass, Channel Equalization, and Coupled Electric Drives tasks, and commented on how the relative difficulty of the tasks affected the results. The ESNs with feedback exhibited outstanding performance improvement in the Channel Equalization and Coupled Electric Drives tasks, and we observed a roughly $57\%$ and $59\%$ improvement in their respective error measures. For the more difficult Mackey-Glass task requiring a 10-step-ahead prediction, we still saw a roughly $30\%$ improvement in the averaged NMSE, showing that feedback produces a significant boost in performance for a variety of tasks. These ESNs with feedback were shown to perform just as well on average, if not better than, the ESNs that have double the number of computational nodes without feedback. 

Our feedback method bares some partial resemblance to the FORCE learning algorithm \cite{SA09} that was discussed earlier in the introduction. One may wonder how this method's performance compares to the results we provide for our feedback procedure, but a fair comparison cannot be made for the tasks presented in this paper. The reason for this is that the ESN with feedback utilizes a control sequence as input, whereas the FORCE learning algorithm runs autonomously without input after training; if there is any input considered in \cite{SA09} they are simply switching signals to designate which output waveform should be generated rather than random inputs such as a PBRS signal. Both the Channel Identification and Coupled Electric Drives tasks involve producing an output based off of an input sequence that is largely uncorrelated in time, so an autonomous method like FORCE learning cannot be applied to these tasks. This leaves the Mackey-Glass task, which only uses prior values of the target sequence as input for the ESN. However, while FORCE learning has been shown to have decent performance on the task with 100 computational nodes \cite{Li:2022}, standard ESNs have been able to perform exceedingly accurate one-step-ahead predictions for this task with the same number of nodes \cite{shahi2022prediction}. We also performed simulations that corroborate the general performance of the above references. Additionally, it is not clear how to appropriately compare the 10-steps-ahead prediction of the Mackey-Glass system that we investigate in this work to FORCE results. Therefore, we also do not believe a comparison of our ESN with feedback performance with FORCE learning results for the Mackey-Glass task will be instructive on the practical use of our feedback procedure.

Although there will be an additional hardware modification required to implement feedback, such a modification will only be external to the reservoir computer and, therefore, the burden will be minimal. This feedback scheme is designed to avoid any direct modification of the ESN's main body (i.e., the computing bulk) since we only need to take the readout of the network and send some component of that readout back into the network with the usual input. Thus, we will only need an apparatus that connects to the readout and the input of the ESN, but does not require modifying the internal reservoir. Given that our results suggest that ESNs with feedback will perform just as well as, if not better than, ESNs of double the number of nodes without feedback, the cost-benefit analysis of adding feedback hardware is very likely to be more favorable than increasing the size of the ESN to achieve similar performance.

Because of the highly complex and nonlinear dependence of the network states $\{x_k\}$ on $V$, we used batch gradient descent for the optimization of $V$, but better methods may very well exist. The question of how to best optimize $V$ is indeed closely related to how to choose the best $(A,B)$ for a given training data set $\{u_k\}$ and $\{y_k\}$. This is an open question for which any progress would be monumental in the general theoretical development of reservoir computing and neural networks. Even providing just a measure of the computing power of a given $(A,B)$ or  $(A,B,\{u_k\})$ outside of the cost function itself could provide some classification of tasks that will save us significant computational time and resources in the future.

\appendix

\section{Batch Gradient Descent Method for Optimizing \texorpdfstring{$V$}{TEXT}} \label{sec: Details of Batch Gradient Descent Method used in Numerical Examples}

At the beginning of each step, we train the ESN using $\{u_k + V_i^T x_k\}$ as the input sequence for the current feedback vector $V_i$, with $V_0=\mathbf{0}$ as the initial step. To get the change in $V$, we use the gradient of $S_{\mathrm{min}}$ given by 
\begin{align}
    \frac{d S}{d V_j} &= \frac{1}{N}\sum_{k=0}^{N-1}\left(W^\top\frac{dx_k}{dV_j}\right)(W^\top x_k-y_k) \\
    \frac{dx_k}{dV_j} &= \Sigma_k\left(B~x_{k-1,j}+\overline{A}\frac{dx_{k-1}}{dV_j}\right) \\
    \Sigma_{k,ij} &= \delta_{ij} \sigma'(z_{k-1,i})=\delta_{ij} x_{k,i}(1-x_{k,i}),
\end{align}
where $\overline{A} = A + B V^\top$. The derivatives $\frac{dx_k}{dV_j}$ can be calculated by iteration starting from the initial condition $\frac{dx_0}{dV_j} = \mathbf{0}$. Then we simply shift $V \rightarrow V-\eta \nabla_V S_{\mathrm{min}}$ for some learning rate $\eta$ at each step of the descent, using the optimal solutions for $C$ and $W$ for the current value of $V$.

\section{Nonlinear Stochastic Autoregressive Model} \label{sec:nonlinear stochastic autoregressive model}

The basic model is
\begin{align*}
x_{k+1} &= g(Ax_k + B\nu (u_k,y_k)) \\
\hat{y}_k &= W^{\top} x_k +C,
\end{align*}
where $\nu$ is some function of $u_k$ and $y_k$  taking values in $\mathbb{R}^n$ and $\{(u_k,y_k)\}$ are I/O pairs generated at time $k$ by the stochastic dynamical system of interest. The quantity $\hat{y}_k$ is the output of the ESN and functions as an approximation of $y_k$. Assuming convergence, after washout $\hat{y}_k$ can be expressed as \cite{Chen:2022}:
$$
\hat{y}_k = F(y_{k-1},y_{k-2},\ldots,u_{k-1},u_{k-2},\ldots),
$$
for some nonlinear functional $F$ of past input and output values, $u_{k-1},u_{k-2},\ldots$  and $y_{k-1},y_{k-2},\ldots$, respectively. Letting $e_k$ be the approximation error of $y_k$ by $\hat{y}_k$, $e_k = y_k - \hat{y_k}$, the ESN then generates a nonlinear infinite-order autoregressive model with exogenous input:
$$
y_k = F(y_{k-1},y_{k-2},\ldots,u_{k-1},u_{k-2},\ldots) + e_k.
$$
To complete the model, it is stipulated that $\{e_k\}$ is Gaussian white noise sequence that is uncorrelated with the input sequence $\{u_k\}$. By making the substitution $y_k = \hat{y}_k+e_k = W^{\top}x_k + C + e_k$, the system identification procedure identifies an approximate state-space model of the unknown stochastic dynamical system given by:
\begin{align*}
x_{k+1} &= f(x_k,u_k,e_k)\\
y_k &= W^{\top}x_k + e_k,
\end{align*}
where 
$$
f(x_k,u_k,e_k) = g(Ax_k + B\nu(u_k, W^{\top}x_k  +e_k)),
$$
where $e_k$ and $u_k$ free inputs to the model. The stochasticity in the model comes from the free noise sequence $\{e_k\}$ and possibly also the input $\{u_k\}$ (typically the case in system identification). This model can then be used, for instance, to design stochastic control laws $u_k$ for the unknown stochastic dynamical system or to simulate it on a digital computer. The hypothesis of the model, that $\{e_k\}$ is a Gaussian white noise sequence that is uncorrelated with the input sequence $\{u_k\}$ is tested after the model is fitted using a separate validation data set (different from the fitting data set) by residual analysis; see, e.g., \cite{Ljung99,Billings13}.

\bibliographystyle{elsarticle-num}
\bibliography{refs}

\begin{thebibliography}{10}
\expandafter\ifx\csname url\endcsname\relax
  \def\url#1{\texttt{#1}}\fi
\expandafter\ifx\csname urlprefix\endcsname\relax\def\urlprefix{URL }\fi
\expandafter\ifx\csname href\endcsname\relax
  \def\href#1#2{#2} \def\path#1{#1}\fi

\bibitem{lukovsevivcius2009reservoir}
M.~Luko{\v{s}}evi{\v{c}}ius, H.~Jaeger, Reservoir computing approaches to recurrent neural network training, Computer science review 3~(3) (2009) 127--149.

\bibitem{schrauwen2007overview}
B.~Schrauwen, D.~Verstraeten, J.~Van~Campenhout, An overview of reservoir computing: theory, applications and implementations, in: Proceedings of the 15th european symposium on artificial neural networks. p. 471-482 2007, 2007, pp. 471--482.

\bibitem{tanaka2019recent}
G.~Tanaka, T.~Yamane, J.~B. H{\'e}roux, R.~Nakane, N.~Kanazawa, S.~Takeda, H.~Numata, D.~Nakano, A.~Hirose, Recent advances in physical reservoir computing: A review, Neural Networks 115 (2019) 100--123.

\bibitem{JH04}
H.~Jaeger, H.~Haas, Harnessing nonlinearity: {P}redicting chaotic systems and saving energy in wireless communications, Science 304 (2004) 5667.

\bibitem{pathak2018model}
J.~Pathak, et~al., Model-free prediction of large spatiotemporally chaotic systems from data: A reservoir computing approach, Physical review letters 120~(2) (2018) 024102.

\bibitem{Rafayelyan20}
M.~Rafayelyan, et~al., Large-scale optical reservoir computing for spatiotemporal chaotic systems prediction, Phys. Rev. X 10 (2020) 041037.

\bibitem{antonik2018using}
P.~Antonik, M.~Gulina, J.~Pauwels, S.~Massar, Using a reservoir computer to learn chaotic attractors, with applications to chaos synchronization and cryptography, Physical Review E 98~(1) (2018) 012215.

\bibitem{LPHGBO17}
Z.~Lu, et~al., Reservoir observers: Model-free inference of unmeasured variables in chaotic systems, Chaos 27 (2017) 041102.

\bibitem{grigoryeva2016reservoir}
L.~Grigoryeva, J.~Henriques, J.-P. Ortega, Reservoir computing: information processing of stationary signals, in: Joint 2016 CSE, EUC and DCABES, IEEE, 2016, pp. 496--503.

\bibitem{Grigoryeva:2018}
L.~Grigoryeva, J.-P. Ortega, Echo state networks are universal, Neural Networks 108 (2018) 495--508.
\newblock \href {https://doi.org/https://doi.org/10.1016/j.neunet.2018.08.025} {\path{doi:https://doi.org/10.1016/j.neunet.2018.08.025}}.

\bibitem{gonon2019reservoir}
L.~Gonon, J.-P. Ortega, Reservoir computing universality with stochastic inputs, IEEE transactions on neural networks and learning systems 31~(1) (2019) 100--112.

\bibitem{Chen:2022}
J.~Chen, H.~I. Nurdin, Nonlinear autoregression with convergent dynamics on novel computational platforms, IEEE Transactions on Control Systems Technology 30~(5) (2022) 2228--2234.
\newblock \href {https://doi.org/10.1109/TCST.2021.3136227} {\path{doi:10.1109/TCST.2021.3136227}}.

\bibitem{larger2017high}
L.~Larger, et~al., High-speed photonic reservoir computing using a time-delay-based architecture: Million words per second classification, Physical Review X 7~(1) (2017) 011015.

\bibitem{torrejon2017neuromorphic}
J.~Torrejon, et~al., Neuromorphic computing with nanoscale spintronic oscillators, Nature 547~(7664) (2017) 428--431.

\bibitem{chen2020temporal}
J.~Chen, H.~I. Nurdin, N.~Yamamoto, Temporal information processing on noisy quantum computers, Phys. Rev. Applied 14 (2020) 024065.
\newblock \href {https://doi.org/10.1103/PhysRevApplied.14.024065} {\path{doi:10.1103/PhysRevApplied.14.024065}}.

\bibitem{Suzuki22}
Y.~Suzuki, et~al., Natural quantum reservoir computing for temporal information processing, Sci. Reports 12~(1) (2022) 1353.

\bibitem{Yasuda23}
T.~Yasuda, et~al., Quantum reservoir computing with repeated measurements on superconducting devices, arXiv preprint arXiv:2310.06706 (October 2023).

\bibitem{nakajima2021reservoir}
K.~Nakajima, I.~Fischer, Reservoir Computing: Theory, Physical Implementations, and Applications, Springer Singapore, 2021.

\bibitem{Mujal21}
P.~Mujal, et~al., Opportunities in quantum reservoir computing and extreme learning machines, Adv. Quantum Technol. 4 (2021) 2100027.

\bibitem{markovic2020quantum}
D.~Markovi{\'c}, J.~Grollier, Quantum neuromorphic computing, Applied Physics Letters 117~(15) (2020) 150501.

\bibitem{jaeger2001echo}
H.~Jaeger, The “echo state” approach to analysing and training recurrent neural networks-with an erratum note, Bonn, Germany: German National Research Center for Information Technology GMD Technical Report 148~(34) (2001) 13.

\bibitem{jaeger2007echo}
H.~Jaeger, Echo state network, scholarpedia 2~(9) (2007) 2330.

\bibitem{lukovsevivcius2012practical}
M.~Luko{\v{s}}evi{\v{c}}ius, A practical guide to applying echo state networks, in: Neural Networks: Tricks of the Trade: Second Edition, Springer, 2012, pp. 659--686.

\bibitem{SA09}
D.~Sussillo, L.~F. Abbott, Generating coherent patterns of activity from chaotic neural networks, Neuron 63~(4) (2009) 544--557.

\bibitem{FBD20}
M.~Freiberger, P.~Bienstman, J.~Dambre, A training algorithm for networks of high-variability reservoirs, Scientific Reports 10 (2020) 14451.

\bibitem{MJS07}
W.~Maass, P.~Joshi, E.~D. Sontag, Computational aspects of feedback in neural circuits, PLoS Comp. Bio. 3 (2007) article no. e165.

\bibitem{Luko07}
M.~Luko\u{s}evi\u{c}ius, Echo state networks with trained feedbacks, {T}ech. Rep. No. 4, School of Engineering and Science, Jacobs University Bremen (2007).

\bibitem{Manjunath:2013}
G.~Manjunath, H.~Jaeger, Echo state property linked to an input: Exploring a fundamental characteristic of recurrent neural networks, Neural Computation 25~(3) (2013) 671--696.
\newblock \href {https://doi.org/10.1162/NECO_a_00411} {\path{doi:10.1162/NECO_a_00411}}.

\bibitem{Boyd:1985}
S.~Boyd, L.~Chua, Fading memory and the problem of approximating nonlinear operators with volterra series, IEEE Transactions on Circuits and Systems 32~(11) (1985) 1150--1161.
\newblock \href {https://doi.org/10.1109/TCS.1985.1085649} {\path{doi:10.1109/TCS.1985.1085649}}.

\bibitem{Tran:2019}
D.~N. Tran, B.~S. Rüffer, C.~M. Kellett, Convergence properties for discrete-time nonlinear systems, IEEE Transactions on Automatic Control 64~(8) (2019) 3415--3422.
\newblock \href {https://doi.org/10.1109/TAC.2018.2879951} {\path{doi:10.1109/TAC.2018.2879951}}.

\bibitem{Fujii:2017}
K.~Fujii, K.~Nakajima, Harnessing disordered-ensemble quantum dynamics for machine learning, Phys. Rev. Appl. 8 (2017) 024030.
\newblock \href {https://doi.org/10.1103/PhysRevApplied.8.024030} {\path{doi:10.1103/PhysRevApplied.8.024030}}.

\bibitem{Kubota:2021}
T.~Kubota, H.~Takahashi, K.~Nakajima, Unifying framework for information processing in stochastically driven dynamical systems, Phys. Rev. Res. 3 (2021) 043135.
\newblock \href {https://doi.org/10.1103/PhysRevResearch.3.043135} {\path{doi:10.1103/PhysRevResearch.3.043135}}.

\bibitem{Nakajima:2019}
K.~Nakajima, K.~Fujii, M.~Negoro, K.~Mitarai, M.~Kitagawa, Boosting computational power through spatial multiplexing in quantum reservoir computing, Phys. Rev. Appl. 11 (2019) 034021.
\newblock \href {https://doi.org/10.1103/PhysRevApplied.11.034021} {\path{doi:10.1103/PhysRevApplied.11.034021}}.

\bibitem{Rudin:1976}
W.~Rudin, Principles of Mathematical Analysis, New York: McGraw-Hill, 1976.

\bibitem{williams1989learning}
R.~J. Williams, D.~Zipser, A learning algorithm for continually running fully recurrent neural networks, Neural computation 1~(2) (1989) 270--280.

\bibitem{robinson1987utility}
A.~J. Robinson, F.~Fallside, The utility driven dynamic error propagation network, Vol.~11, University of Cambridge Department of Engineering Cambridge, 1987.

\bibitem{williams1995gradient}
R.~J. Williams, D.~Zipser, Gradient-based learning algorithms for recurrent networks and their computational complexity. (1995).

\bibitem{rumelhart1986learning}
D.~E. Rumelhart, G.~E. Hinton, R.~J. Williams, Learning representations by back-propagating errors, nature 323~(6088) (1986) 533--536.

\bibitem{werbos1990backpropagation}
P.~J. Werbos, Backpropagation through time: what it does and how to do it, Proceedings of the IEEE 78~(10) (1990) 1550--1560.

\bibitem{Hulser:2023}
T.~Hülser, F.~Köster, K.~Lüdge, L.~Jaurigue, Deriving task specific performance from the information processing capacity of a reservoir computer, Nanophotonics 12~(5) (2023) 937--947.
\newblock \href {https://doi.org/doi:10.1515/nanoph-2022-0415} {\path{doi:doi:10.1515/nanoph-2022-0415}}.

\bibitem{Wigren:2017}
T.~Wigren, M.~Schoukens, Coupled electric drives data set and reference models, Tech. rep., Department of Information Technology, Uppsala University, Uppsala, Sweden (2017).

\bibitem{Ljung99}
L.~Ljung, System Identification: Theory for the User, 2nd Edition, Prentice-Hall, 1999.

\bibitem{Billings13}
S.~A. Billings, Nonlinear System Identification: NARMAX Methods in the Time, Frequency, and Spatio-Temporal Domains, Wiley, 2013.

\bibitem{Lilliefors:1967}
H.~W. Lilliefors, On the {K}olmogorov-{S}mirnov test for normality with mean and variance unknown, Journal of the American Statistical Association 62~(318) (1967) 399--402.

\bibitem{Abdi:2007}
H.~Abdi, P.~Molin, Lilliefors/{V}an {S}oest's test of normality, Encyclopedia Meas. Stat. (01 2007).

\bibitem{Wilk:1968}
M.~B. Wilk, R.~Gnanadesikan, \href{https://doi.org/10.1093/biomet/55.1.1}{{Probability plotting methods for the analysis for the analysis of data}}, Biometrika 55~(1) (1968) 1--17.
\newblock \href {http://arxiv.org/abs/https://academic.oup.com/biomet/article-pdf/55/1/1/730568/55-1-1.pdf} {\path{arXiv:https://academic.oup.com/biomet/article-pdf/55/1/1/730568/55-1-1.pdf}}, \href {https://doi.org/10.1093/biomet/55.1.1} {\path{doi:10.1093/biomet/55.1.1}}.
\newline\urlprefix\url{https://doi.org/10.1093/biomet/55.1.1}

\bibitem{Li:2022}
Y.~Li, K.~Hu, K.~Nakajima, Y.~Pan, Composite force learning of chaotic echo state networks for time-series prediction, in: 2022 41st Chinese Control Conference (CCC), 2022, pp. 7355--7360.
\newblock \href {https://doi.org/10.23919/CCC55666.2022.9901897} {\path{doi:10.23919/CCC55666.2022.9901897}}.

\bibitem{shahi2022prediction}
S.~Shahi, F.~H. Fenton, E.~M. Cherry, Prediction of chaotic time series using recurrent neural networks and reservoir computing techniques: A comparative study, Machine learning with applications 8 (2022) 100300.

\end{thebibliography}

\end{document}